%% file: root.tex
\documentclass[letterpaper, 10 pt, conference]{ieeeconf}  

\IEEEoverridecommandlockouts                              

\input{packages}

\input{macros}

\title{\LARGE \bf
A Ray Intersection Algorithm for Fast Growth Distance Computation Between Convex Sets
}

\author{%
Akshay Thirugnanam and Koushil Sreenath
\thanks{%
This work was supported in part by the David and Patricia Bogy Fellowship within the Department of Mechanical Engineering at UC Berkeley.
}
\thanks{%
The authors are with the Department of Mechanical Engineering, UC Berkeley, CA 94720, USA. \tt\small\{akshay\_t, koushils\}@berkeley.edu
}%
}

\begin{document}

\maketitle
\pagestyle{plain} 
\pagenumbering{arabic}

\input{sections/abstract}

\input{sections/introduction}
\input{sections/background}
\input{sections/growth-distance-algorithm}
\input{sections/results}
\input{sections/conclusions}





\section*{ACKNOWLEDGMENT}

The authors thank Prof. Ong Chong Jin at the National University of Singapore and Jeongmin Lee at Seoul National University for the implementations of the Incremental and the DCF algorithms, respectively.


\bibliographystyle{IEEEtran}
\bibliography{IEEEabrv,references}

\end{document}

%% file: packages.tex
\usepackage{cite}

\usepackage{xurl} 
\usepackage{hyperref}
\hypersetup{
    breaklinks=true,   
}

\usepackage{bm} 

\usepackage{amsmath,amssymb,amsfonts}
\usepackage{mathtools} 
\usepackage{amsthm} 
\usepackage[bb=dsserif]{mathalpha} 
\usepackage{eqparbox,array}

\usepackage{siunitx}

\usepackage{tabularx,ragged2e}
\newcolumntype{L}{>{\raggedright\arraybackslash}X} 
\newcolumntype{C}{>{\Centering\arraybackslash}X} 
\newcolumntype{R}{>{\raggedleft\arraybackslash}X} 
\usepackage{multirow}
\usepackage{booktabs}
\usepackage{array}
\usepackage{arydshln} 

\usepackage{graphicx}
\usepackage[caption=false]{subfig} 
\usepackage{xcolor} 

\usepackage{algorithmic} 
\usepackage{algorithm}

\newcommand\LONGCOMMENT[2]{%
  \hfill$\rhd$ \begin{minipage}[t]{#1}#2\strut\end{minipage}%
}
\newcommand\LONGTEXT[2]{%
  \hfill\begin{minipage}[t]{#1}#2\strut\end{minipage}%
}

\usepackage{stfloats}

\usepackage{float}
\let\labelindent\relax
\usepackage{enumitem}

\usepackage{datetime}

\usepackage{pbalance}

\newtheorem{proposition}{Proposition}
\newtheorem{theorem}{Theorem}
\newtheorem{lemma}{Lemma}
\newtheorem{remark}{Remark}
\newtheorem{definition}{Definition}

\newtheorem{problem}{Problem}

%% file: macros.tex
\newcommand{\real}{\mathbb{R}}

\newcommand{\realnn}{\mathbb{R}_\texttt{+}}
\newcommand{\realp}{\mathbb{R}_\texttt{++}}

\DeclareMathOperator{\setcl}{cl}

\DeclareMathOperator{\setint}{int}
\DeclareMathOperator{\setbd}{bd}
\DeclareMathOperator{\dom}{dom}
\DeclareMathOperator{\conv}{conv}
\DeclareMathOperator{\cone}{cone}
\DeclareMathOperator*{\argmin}{argmin} 
\DeclareMathOperator*{\argmax}{argmax} 



%% file: sections/abstract.tex
\begin{abstract}
\label{abstract}
In this paper, we discuss an efficient algorithm for computing the growth distance between two compact convex sets with representable support functions.
The growth distance between two sets is the minimum scaling factor such that the sets intersect when scaled about some center points.
Unlike the minimum distance between sets, the growth distance provides a unified measure for set intersection and separation.
We first reduce the growth distance problem to an equivalent ray intersection problem on the Minkowski difference set.
Then, we propose an algorithm to solve the ray intersection problem by iteratively constructing inner and outer polyhedral approximations of the Minkowski difference set.
We show that our algorithm satisfies several key properties, such as primal and dual feasibility and monotone convergence.
We provide extensive benchmark results for our algorithm and show that our open-source implementation achieves state-of-the-art performance across a wide variety of convex sets.
Finally, we demonstrate robotics applications of our algorithm in motion planning and rigid-body simulation.

\end{abstract}

%% file: sections/introduction.tex
\section{Introduction}
\label{sec:introduction}

Collision detection and distance computation between sets are essential problems with applications in collision avoidance for robot safety~\cite{gilbert2003distance,stasse2008real,schulman2014motion,mirabel2016hpp}, computer graphics~\cite{van2003collision}, and physics simulators~\cite{todorov2012mujoco}.
Collision detection in real-world applications is a challenging problem because the number of possible collisions scales quadratically with the number of sets, and the sets can be nonconvex.
Many existing collision detection pipelines speed up computation by eliminating pairs of sets that are far apart in the \textit{broadphase} step~\cite{ericson2004real} and by hierarchically decomposing nonconvex sets into convex ones~\cite{wei2022coacd}.
The final step, called the \textit{narrowphase}, involves distance computation between convex sets, which is a convex optimization problem.

The Gilbert, Johnson, and Keerthi (GJK) algorithm~\cite{gilbert1987fast} is a popular algorithm for computing the minimum distance between convex sets using their support functions.
Many improvements to the GJK algorithm have been discussed in the literature, including the enhanced-GJK algorithm~\cite{cameron1997enhancing,gilbert2002fast}, robust distance subproblem computation~\cite{montanari2017improving}, and Nesterov and Polyak accelerated GJK~\cite{montaut2024gjk}. 
A drawback of the minimum distance function for convex sets is that, when the sets intersect, it provides no information about the extent of penetration.
Algorithms such as the Expanding Polytope Algorithm (EPA)~\cite{van2001proximity} compute the penetration depth between intersecting convex sets, but are often slower because they solve a nonconvex optimization problem.
By combining the GJK algorithm and EPA, the signed distance between convex sets can be computed as the difference between the minimum distance and the penetration depth.

The growth distance function, introduced in \cite{ong1996growth}, provides an alternative measure of convex set separation and intersection.
The growth distance between two sets is defined as the minimum scaling factor $\alpha^*$ such that the sets intersect when scaled about some given center points.
A key property of the growth distance function is that it provides a unified measure for convex set intersection ($\alpha^* \leq 1$) and separation ($\alpha^* > 1$).
Furthermore, the growth distance between convex sets can be computed as the solution to a convex optimization problem (see Sec.~\ref{sec:background}), unlike the penetration depth, which requires a nonconvex optimization problem.
Thus, for robotics applications where general distance measures between convex sets can be used, such as motion planning and trajectory optimization, computing the growth distance can be more advantageous than the signed distance.
In this paper, we propose an algorithm similar to the GJK algorithm to compute the growth distance between convex sets.

\subsection{Related Works}
\label{subsec:related-works}

\begin{table*}[!t]
    \footnotesize
    \caption{Comparison of Growth Distance Methods}
    \label{tab:comparison-growth-distance-methods}
    \begin{tabularx}{\textwidth}{>{\raggedright\arraybackslash}p{2.3cm}CCCCC}
    \toprule
    Property & Internal Expanding procedure (\emph{IE}), \cite{zheng2010fast} & Differentiable Contact Features (\emph{DCF}), \cite{lee2023uncertain} & Incremental algorithm (\emph{Inc}), \cite{ong2000fast} & Differentiable Collisions (\emph{DCol}), \cite{tracy2023differentiable} & \textbf{Our method} \\
    \midrule
    Types of convex sets & Convex primitives\footref{footnote:convex-primitives}, vertex meshes, DSFs & Differentiable Support Functions (DSFs) & Facet meshes & Convex primitives & Convex primitives, vertex meshes, DSFs \\
    \addlinespace[5pt]
    Algorithm outputs & Primal and dual optimal solutions & Primal and dual optimal solutions & Primal and dual optimal solutions & Primal optimal solution & Primal and dual optimal solutions \\
    \addlinespace[5pt]
    Solution guarantees & Primal and dual feasibility & Dual feasibility & Primal feasibility, compl. slackness & Dual feasibility & Primal and dual feasibility \\
    \addlinespace[5pt]
    Termination criterion & Primal-dual gap & Dual optimality error (gradient norm) & Dual infeasibility error & KKT error & Primal-dual gap \\
    \addlinespace[5pt]
    Warm start support & No & No & Yes & No & Yes \\
    \bottomrule
\end{tabularx}
\end{table*}

The growth distance problem for compact convex sets was introduced in \cite{ong1996growth}, which used a Linear Program (LP) solver to compute the growth distance between polytopes.
Subsequent works accelerated the LP solver by considering adjacency information between the facets of the polytopes~\cite{hong2000fast,ong2000fast}.
The Incremental (\emph{Inc}) algorithm was also proposed in \cite{ong2000fast} to warm-start the computation of growth distance for polytopes.

The growth distance problem between convex sets is known to be equivalent to a ray intersection problem for the Minkowski difference of the two convex sets, see \cite{zheng2013ray}.
Subsequently, ray intersection algorithms have also been used to compute the growth distance between convex sets.
A GJK-based ray intersection algorithm was proposed in \cite{van2004ray}, based on which \cite{zheng2009numerical} used a conic projection subproblem to solve the ray intersection problem.
The Internal Expanding (\emph{IE}) algorithm was discussed in \cite{zheng2010fast}, and later combined with the projection method in \cite{zheng2009numerical} to form the hybrid \emph{ZC-IE} algorithm in \cite{zheng2013ray}.

Previous works have also considered optimization-based methods for computing the growth distance.
The \emph{DCol} algorithm~\cite{tracy2023differentiable} uses a conic solver to compute the growth distance between small convex primitives.
Likewise, the Differentiable Contact Features (\emph{DCF}) method~\cite{lee2023uncertain} computes the growth distance between Differentiable Support Function (DSF) sets, which are convex sets with twice continuously differentiable support functions, using a hybrid \emph{IE}-based Newton solver.

\subsection{Contributions}
\label{subsec:contributions}

In this paper, we address the growth distance problem by solving the equivalent ray intersection problem.
The contributions of our paper are as follows:

\noindent (i) We propose an LP-based inner-outer polyhedral approximation method for solving the ray intersection problem, in contrast to the \emph{IE} algorithm, which maintains and updates a list of facets.
We provide guarantees on primal and dual feasibility, as well as on the monotonicity of the primal-dual gap for our proposed growth distance algorithm.
Additionally, we describe extensions to our algorithm that enable warm-starting and collision detection.

\noindent (ii) We provide thorough benchmarks comparing our algorithm to existing open-source methods and optimized implementations of closed-source methods. 
We show that our algorithm outperforms all existing methods for 2D and 3D convex primitives and mesh sets (except the \emph{DCF} method for DSF sets), achieving computation times on the order of one microsecond.
Consequently, our algorithm allows speeding up robotic applications that rely on distance computation and collision detection.
We note that, to limit the scope of this paper, we do not compare the \emph{IE} algorithm to our algorithm for dimensions greater than three.
A comparison of the existing methods for growth distance computation is shown in Tab.~\ref{tab:comparison-growth-distance-methods}.

\noindent (iii) We provide an open-source C\texttt{++} implementation of our algorithm with warm start and collision detection functionality.
We demonstrate the applicability of our implementation in two robotics applications: motion planning and rigid-body simulation.
Finally, we highlight the advantages and disadvantages of the growth distance in comparison to the signed distance.

\subsection{Notations}
\label{subsec:notation}

$\real$, $\realnn$, and $\realp$ are the sets of real, nonnegative real, and positive real numbers, respectively.
$\setint{\mathcal{C}}$, $\setcl{\mathcal{C}}$, and $\setbd{\mathcal{C}}$ are the interior, closure, and boundary of the set $\mathcal{C}$, respectively.
$\dom{f}$ is the domain of the function $f$.
Throughout the paper, we will work with the Euclidean vector space $\real^n$. 
$\mathbb{1}_n \in \real^n$ and $\mathbb{0}_n \in \real^n$ are the vectors with all components as one and zero, respectively.
Dimensions are omitted when they can be contextually inferred.
$\langle \cdot, \cdot \rangle$ denotes an inner product on $\real^n$, and $\lVert \cdot \rVert$ is the corresponding norm.
For $x \in \real^n$, $B_r(x) = \{y: \lVert y-x\rVert \leq r\}$, and $S^{n-1} = \setbd(B_1(\mathbb{0}_n))$.
$GL(n)$, $SO(n)$, and $SE(n)$ are the general linear, special orthogonal, and special Euclidean groups, respectively.
$\conv{\mathcal{S}}$ and $\cone{\mathcal{S}}$ represent the convex hull and (convex) conic hull of $\mathcal{S} \subset \real^n$.

%% file: sections/background.tex
\section{Background}
\label{sec:background}

In this section, we define proper convex sets, which constitute the inputs for the growth distance algorithm in Sec.~\ref{sec:growth-distance-algorithm}, the support function of a convex set, and some algebraic operations on convex sets.
We also state the growth distance problem and its equivalence to a ray intersection problem, which will be the starting point for the growth distance algorithm in Sec.~\ref{sec:growth-distance-algorithm}.

\subsection{Convex Set Definitions}
\label{subsec:convex-set-definitions}

In this paper, we will primarily focus on the growth distance between proper convex sets, which are compact convex sets with nonzero volume and comprise a large class of objects considered in robotics applications.

\begin{definition} (Proper convex set)
\label{def:proper-convex-set}
A convex set $\mathcal{C} \subset \real^l$ is called a proper convex (PC) set if it is compact and solid, i.e., $\mathcal{C}$ is closed, bounded, and has a nonempty interior.
\end{definition}

A proper convex (PC) set $\mathcal{C}$ is also a \textit{regular set} (i.e., $\setcl(\setint{\mathcal{C}}) = \mathcal{C}$), which is a common topological assumption used in path planning problems~\cite{rodriguez2012path}.
The significance of PC sets for the growth distance problem will be elaborated upon in Sec.~\ref{subsec:growth-distance-between-convex-sets}.
Next, we define the inradius of a compact convex set $\mathcal{C}$ at a point $p \in \mathcal{C}$.

\begin{definition} (Inradius)
\label{def:inradius}
Let $\mathcal{C} \neq \emptyset$ be a compact convex set and $p \in \mathcal{C}$.
Then, the inradius $r(\mathcal{C}, p)$ of $\mathcal{C}$ at $p$ is given by
\begin{equation}
\label{eq:inradius}
r(\mathcal{C}, p) := \max \{r \geq 0: B_r(p) \subset \mathcal{C}\}.
\end{equation}
\end{definition}
Note that for a compact convex set $\mathcal{C} \ni p$, $r(\mathcal{C}, p) \geq 0$, and if $\mathcal{C}$ is a PC set and $p \in \setint{\mathcal{C}}$, then $r(\mathcal{C}, p) > 0$.

\subsection{Support Function of a Convex Set}
\label{subsec:support-function-of-a-convex-set}

For a nonempty compact convex set $\mathcal{C} \subset \real^l$, the \textit{support function} $s_v[\mathcal{C}]: \real^l \rightarrow \real$ and \textit{support patch map} $S_p[\mathcal{C}]: \real^l \rightrightarrows \real^l$ are defined as (see \cite[Sec.~13]{rockafellar1997convex})
\begin{subequations}
\label{eq:support-function-support-patch-map}
\begin{align}
s_v[\mathcal{C}](\lambda) & = \max_{z \in \mathcal{C}} \ \langle \lambda, z\rangle, \label{subeq:support-function} \\
S_p[\mathcal{C}](\lambda) & = \argmax_{z \in \mathcal{C}} \ \langle \lambda, z\rangle. \label{subeq:support-patch-map}
\end{align}
\end{subequations}
The subscripts `$v$' and `$p$' in $s_v[\mathcal{C}]$ and $S_p[\mathcal{C}]$ stand for `value' and `patch'.
The vector $\lambda \in \real^l$ is the direction along which the support function is computed.
Since $\mathcal{C}$ is compact, $\dom(s_v[\mathcal{C}]) = \dom(S_p[\mathcal{C}]) = \real^l$, $s_v[\mathcal{C}]$ is positively homogeneous~\cite[Cor.~13.2.2]{rockafellar1997convex}, and $S_p[\mathcal{C}]$ is pointwise compact.
We can define a \textit{support point function} $s_p[\mathcal{C}]: \real^l \rightarrow \real^l$ (here `$p$' stands for `point'), using a selection function that selects an element from the set $S_p[\mathcal{C}]$ for each $\lambda$, as follows:
\begin{equation}
\label{eq:support-point-function}
s_p[\mathcal{C}](\lambda) \in S_p[\mathcal{C}](\lambda).
\end{equation}
Note, from \eqref{eq:support-function-support-patch-map}, that for all $\lambda \in \real^l$,
\begin{equation*}
    s_v[\mathcal{C}](\lambda) = \langle \lambda, s_p[\mathcal{C}](\lambda)\rangle.
\end{equation*}
We do not make any assumptions about the selection, and our algorithm remains valid even if the support points are randomly or nondeterministically selected.

The compact convex set $\mathcal{C} \subset \real^l$ is completely defined by its support function $s_v[\mathcal{C}]$ as (see \cite[Sec.~13]{rockafellar1997convex})
\begin{equation}
    \mathcal{C} = \{z: \langle \lambda, z\rangle \leq s_v[\mathcal{C}](\lambda), \ \forall \lambda \in \real^l\}.
\end{equation}

For example, if the convex set $\mathcal{C}$ is a polytope defined as $\conv\{z_i: i \in \{1, ..., N_v\}\}$, the support and support point functions can be computed as
\begin{align*}
    s_v[\mathcal{C}](\lambda) & = \max_{i \in \{1, ..., N_v\}} \ \langle \lambda, z_i\rangle, \\
    s_p[\mathcal{C}](\lambda) & = z_{i^*}, \quad i^* \in \argmax_{i \in \{1, ..., N_v\}} \ \langle \lambda, z_i\rangle.
\end{align*}

We will use the equivalent support function representation of compact convex sets for the growth distance algorithm in Sec.~\ref{sec:growth-distance-algorithm}, since the support function for many commonly used shapes in robotics, such as convex primitives\footnote{
Convex primitive sets include, among others, `small' polytopes (with $\lessapprox 50$ vertices), cones, frustums, capsules, spheres, ellipsoids, and cylinders.
We distinguish between small polytopes and large polytopes (termed meshes) as the support function computation for large polytopes can benefit from the hill-climbing algorithm~\cite{van2003collision} (see Sec.~\ref{sec:results}).
\label{footnote:convex-primitives}
}, meshes, and point clouds can be easily computed.

\subsection{Affine Transformation and Minkowski Sum}
\label{subsec:affine-transformation-and-minkowski-sum}

In this subsection, we describe two operations on convex sets and the corresponding support function transformations.

\subsubsection{Affine Transformation}
\label{subsubsec:affine-transformation}

Let $g = (b, A) \in \real^l \times GL(l)$ correspond to the \textit{(invertible) affine transformation} $T_g: z \mapsto Az + b$. 
The action of $g$ on a set $\mathcal{C} \subset \real^l$ results in the set $T_g(\mathcal{C})$ given by
\begin{subequations}
\label{eq:affine-transformation}
\begin{align}
T_g(\mathcal{C})  & \ = \{Az + b: z \in \mathcal{C}\}, \\
T_g^{-1}(\mathcal{C}) & \ = \{A^{-1}(z - b): z \in \mathcal{C}\}.
\end{align}
\end{subequations}
For a nonempty compact convex set $\mathcal{C} \subset \real^l$, the support function and support patch map of $T_g(\mathcal{C})$ are given by \cite[Cor.~16.3.1]{rockafellar1997convex} as
\begin{subequations}
\label{eq:affine-transformation-support-function}
\begin{align}
s_v[T_g(\mathcal{C})](\lambda) & = s_v[\mathcal{C}](A^\top \lambda) + \langle \lambda, b\rangle, \label{subeq:affine-transformation-support-function-value} \\
S_p[T_g(\mathcal{C})](\lambda) & = T_g(S_p[\mathcal{C}](A^\top \lambda)). \label{subeq:affine-transformation-support-point-map}
\end{align}
\end{subequations}
Notable examples of affine transformations include:
\begin{align*}
z & \mapsto Rz + p, \tag{\textit{Rigid body transformation}} \\
z & \mapsto sz, \tag{\textit{Scaling transformation}} \\
z & \mapsto (I - 2aa^\top)z, \tag{\textit{Reflection transformation}}
\end{align*}
where $(p, R) \in SE(l)$ is a rigid body transform, $s \in \realp$ is a scaling factor, and $a \in S^{l-1}$ is a reflection axis.

\subsubsection{Minkowski Sum}
\label{subsubsec:minkowski-sum}

The \textit{Minkowski sum} of two nonempty compact convex sets $\mathcal{C}^1$ and $\mathcal{C}^2$ is defined as
\begin{equation}
\label{eq:minkowski-sum}
\mathcal{C}^1 + \mathcal{C}^2 = \{z^1 + z^2: z^1 \in \mathcal{C}^1, z^2 \in \mathcal{C}^2\},
\end{equation}
with the support function and support patch map given by \cite[Sec.~13]{rockafellar1997convex} as
\begin{subequations}
\label{eq:minkowski-sum-support-function}
\begin{align}
s_v[\mathcal{C}^1 + \mathcal{C}^2](\lambda) & = s_v[\mathcal{C}^1](\lambda) + s_v[\mathcal{C}^2](\lambda), \label{subeq:minkowski-sum-support-function-value} \\
S_p[\mathcal{C}^1 + \mathcal{C}^2](\lambda) & = S_p[\mathcal{C}^1](\lambda) + S_p[\mathcal{C}^2](\lambda). \label{subeq:minkowski-sum-support-point-map}
\end{align}
\end{subequations}
We will also define the Minkowski difference between $\mathcal{C}^1$ and $\mathcal{C}^2$ as $\mathcal{C}^1 - \mathcal{C}^2 := \{z^1 - z^2: z^1 \in \mathcal{C}^1, z^2 \in \mathcal{C}^2\}$.


Finally, we show that the set of PC sets is closed under affine transformations and the Minkowski sum.

\begin{lemma}
\label{lem:properties-affine-transformation-minkowski-sum}
Let $\mathcal{C}^1$ be a PC set, $\mathcal{C}^2 \neq \emptyset$ be a compact convex set, and $g \in \real^l \times GL(l)$ be an affine transformation.
Then, $T_g(\mathcal{C}^1)$ and $\mathcal{C}^1 + \mathcal{C}^2$ are PC sets.
\end{lemma}

\begin{proof}
Since $T_g$ is a (continuous) invertible affine transformation, it preserves compact sets, open sets, and convex sets.
Thus, the transformed set $T_g(\mathcal{C}^1)$ is also a PC set.

The set $\mathcal{C}^1 \times \mathcal{C}^2$ is compact and convex, and so its image under the (continuous) linear map $A: (z^1, z^2) \mapsto z^1 + z^2$, given by $\mathcal{C}^1 + \mathcal{C}^2$, is compact and convex.
Additionally, since $\setint(\mathcal{C}^1 + \mathcal{C}^2) \supset \setint(\mathcal{C}^1) + \mathcal{C}^2 \neq \emptyset$, the Minkowski sum $\mathcal{C}^1 + \mathcal{C}^2$ is a PC set.
\end{proof}

\subsection{Growth Distance between Convex Sets}
\label{subsec:growth-distance-between-convex-sets}

In this subsection, we state the growth distance problem and discuss its equivalence to a ray intersection problem.
For the discussion in the rest of the paper, all sets are assumed to be subsets of $\real^l$, unless specified otherwise.

\begin{definition} (Growth distance)
\label{def:growth-distance}
Let $\mathcal{C}^1$ be a PC set and $\mathcal{C}^2 \neq \emptyset$ be a compact convex set with $p^1 \in \setint(\mathcal{C}^1)$ and $p^2 \in \mathcal{C}^2$.
Then, the growth distance $\alpha^*$ between $\mathcal{C}^1$ and $\mathcal{C}^2$ (w.r.t. $p^1$ and $p^2$) is the optimal value of the following problem:
\begin{subequations}
\label{eq:growth-distance-definition}
\begin{align}
\alpha^* = \ \min_{\substack{\alpha \geq 0, z^1, z^2}} \ & \alpha, \\
\text{s.t.} \
& z^1 \in \mathcal{C}^1, \ z^2 \in \mathcal{C}^2, \\
& \alpha (z^1 - p^1) + p^1 = \alpha (z^2 - p^2) + p^2. \label{subeq:growth-distance-definition-intersection}
\end{align}
\end{subequations}
\end{definition}

We call the point $p^1$ about which the set $\mathcal{C}^1$ is scaled as the \textit{center point} of $\mathcal{C}^1$.
For $i \in \{1, 2\}$, when $\mathcal{C}^i$ is scaled by the factor $\alpha$ about the center point $p^i$, $z^i \in \mathcal{C}^i$ is mapped to $\alpha (z^i - p^i) + p^i$.
Thus, the constraint \eqref{subeq:growth-distance-definition-intersection} requires that $\mathcal{C}^1$ and $\mathcal{C}^2$ have a nonempty intersection when scaled by the factor $\alpha$ about their center points.
An example of the computation of the growth distance for two PC sets is shown in Fig.~\ref{subfig:growth-distance-problem}.
Under the assumptions of Def.~\ref{def:growth-distance}, the growth distance satisfies $\alpha^* > 0$ and provides a unified measure for object separation and collision, corresponding to $\alpha^* > 1$ and $\alpha^* \leq 1$, respectively.

\begin{remark} (Existence of optimal solutions)
\label{rem:existence-optimal-solutions}
In Def.~\ref{def:growth-distance}, if $\mathcal{C}^1$ and $\mathcal{C}^2$ are general nonempty closed convex sets, then \eqref{eq:growth-distance-definition} can be infeasible.
Even when the infimum of \eqref{eq:growth-distance-definition} exists, it is possible that an optimal solution to \eqref{eq:growth-distance-definition} does not exist, since the feasible set may be noncompact.
Under the assumptions of Def.~\ref{def:growth-distance}, \eqref{eq:growth-distance-definition} has a continuous cost function and a compact feasible set that is nonempty (since $\mathcal{C}^1$ is a PC set); thus, an optimal solution to \eqref{eq:growth-distance-definition} always exists.
We note that the assumptions of Def.~\ref{def:growth-distance} only require both convex sets to be compact and at least one convex set to be solid.
Thus, in most robotics applications, with the appropriate modeling of convex geometries and selection of center points, the growth distance problem \eqref{eq:growth-distance-definition} is well-defined.
\end{remark}

The growth distance problem \eqref{eq:growth-distance-definition} is equivalent to a ray intersection problem for the Minkowski difference set $\mathcal{C}^1 - \mathcal{C}^2$, as stated in the following result (the proof follows directly from the transformation stated in the result).

\begin{proposition} (Ray intersection problem)~\cite[Sec.~V.A]{zheng2013ray}
\label{prop:ray-intersection-problem}
Let the assumptions of Def~\ref{def:growth-distance} hold and let $p^2 \neq p^1$.
Consider the following convex optimization problem:
\begin{subequations}
\label{eq:ray-intersection-definition}
\begin{align}
\beta^* = \ \max_{\beta, z} \ & \beta, \label{subeq:ray-intersection-definition-value} \\
\text{s.t.} \
& z \in \mathcal{C}, \quad z = \beta p, \label{subeq:ray-intersection-definition-constraints}
\end{align}
\end{subequations}
where $p := p^2 - p^1$, and $\mathcal{C} := \mathcal{C}^1 - \mathcal{C}^2 + \{p\}$.
Then, $\beta^* > 0$ and there is a unique optimal solution $(\beta^*, z^*)$ to \eqref{eq:ray-intersection-definition}.
Further, for $\beta > 0$, $(\beta, z)$ is feasible (optimal) for \eqref{eq:ray-intersection-definition} if and only if $(1/\beta, z^1, z^2)$ is feasible (optimal) for \eqref{eq:growth-distance-definition} for some $z^{1} \in \mathcal{C}^1$ and $z^{2} \in \mathcal{C}^2$, with $z = z^{1} - z^{2} + p$.
\end{proposition}

For the growth distance example in Fig.~\ref{subfig:growth-distance-problem}, Fig.~\ref{subfig:ray-intersection-problem} shows the equivalent ray intersection problem.
Under the assumptions of Def.~\ref{def:growth-distance}, $\mathcal{C}$ is a PC set, and we can write the support function and support patch map of $\mathcal{C}$ using \eqref{eq:affine-transformation} and \eqref{eq:minkowski-sum-support-function} as
\begin{subequations}
\label{eq:minkowski-difference-support-function}
\begin{align}
s_v[\mathcal{C}](\lambda) & = s_v[\mathcal{C}^1](\lambda) + s_v[\mathcal{C}^2](-\lambda) + \langle \lambda, p\rangle, \\
S_p[\mathcal{C}](\lambda) & = S_p[\mathcal{C}^1](\lambda) - S_p[\mathcal{C}^2](-\lambda) + \{p\}.
\end{align}
\end{subequations}

Next, we define $\epsilon$-optimal solutions for \eqref{eq:growth-distance-definition} and \eqref{eq:ray-intersection-definition}, which comprise the outputs of the growth distance algorithm.

\begin{definition} ($\epsilon$-optimal solution)
\label{def:epsilon-optimal-solution}
Under the assumptions of Def.~\ref{def:growth-distance}, a feasible solution $(\alpha, z^1, z^2)$ for \eqref{eq:growth-distance-definition} is $\epsilon$-optimal (for some $\epsilon > 0$) if $\alpha / \alpha^* - 1 \leq \epsilon$.
Likewise, for \eqref{eq:ray-intersection-definition}, a feasible solution $(\beta, z)$ with $\beta > 0$ is $\epsilon$-optimal if $\beta^* / \beta -1 \leq \epsilon$.
\end{definition}

Note, from \eqref{eq:ray-intersection-definition}, that the ratio $\beta^* / \beta$ is independent of the distance between the center points $\lVert p\rVert$.
Thus, the $\epsilon$-optimal solution definition in Def.~\ref{def:epsilon-optimal-solution} allows for a larger absolute error $\beta^* - \beta$ when the two sets are farther apart, and is a more intrinsic metric for convergence.
We can now state the growth distance problem, which will be the focus of Sec.~\ref{sec:growth-distance-algorithm}.

\begin{problem} (Growth distance problem)
\label{problem:growth-distance}
Let $\mathcal{C}^1$ be a PC set and $\mathcal{C}^2 \neq \emptyset$ be a compact convex set with $p^1 \in \setint(\mathcal{C}^1)$ and $p^2 \in \mathcal{C}^2$.
Let the support functions $s_v[\mathcal{C}^i]$ and support point functions $s_p[\mathcal{C}^i]$ be known for $i \in \{1, 2\}$.
Then, compute an $\epsilon_{tol}$-optimal solution to \eqref{eq:growth-distance-definition} for some $\epsilon_{tol} > 0$.
\end{problem}

\begin{figure}[!t]
\centering
\subfloat[
The growth distance problem, \eqref{eq:growth-distance-definition}, for two PC sets $\mathcal{C}^1$ and $\mathcal{C}^2$ (in dark gray).
The growth distance $\alpha^*$ is the minimum scaling factor $\alpha$ such that the sets $\mathcal{C}^1$ and $\mathcal{C}^2$, when scaled about the center positions $p^1$ and $p^2$ by $\alpha$, intersect.
The scaled sets are depicted in light gray.
The two points $z^{1*} \in \mathcal{C}^1$ and $z^{2*} \in \mathcal{C}^2$ (in red) comprise the optimal solution for the growth distance problem.
]{%
\begin{minipage}{\columnwidth}
    \centering
    \includegraphics[width=0.90\columnwidth]{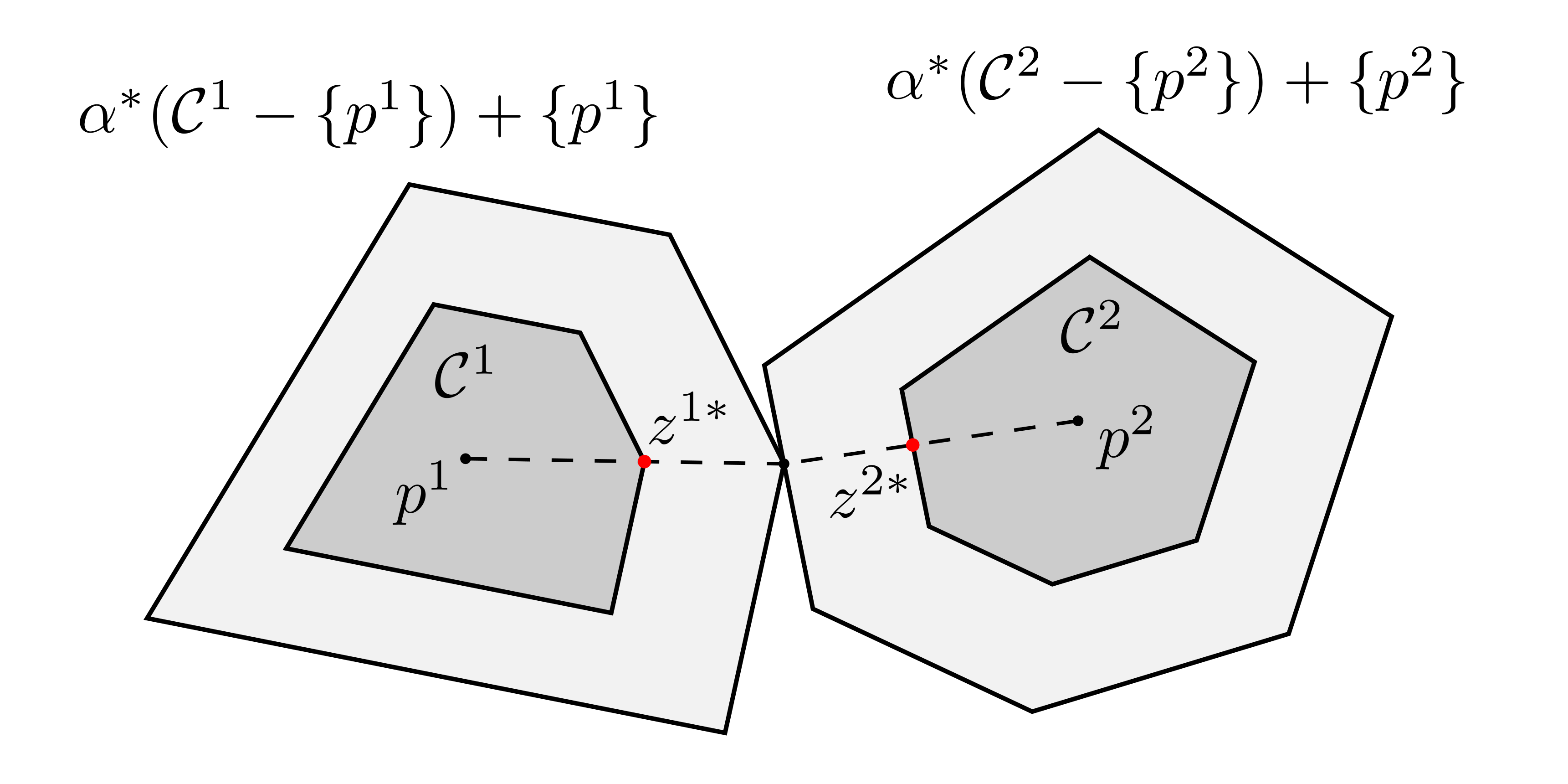}%
    \label{subfig:growth-distance-problem}
\end{minipage}%
}
\hfil
\subfloat[
The ray intersection problem, \eqref{eq:ray-intersection-definition}, corresponding to the growth distance problem depicted in Fig.~\ref{subfig:growth-distance-problem}.
The Minkowski difference set $\mathcal{C}$ (in light gray) is given by $\mathcal{C}^1 - \mathcal{C}^2 + \{p\}$, where $p = p^2 - p^1$ is the relative center position.
The ray intersection point $z^*$ (in red) is the point at which the ray $\vec{\mathbb{0}p}$ intersects with the set boundary $\setbd{\mathcal{C}}$.
The intersection point is unique because $\mathcal{C}$ is a PC set and $\mathbb{0} \in \setint{\mathcal{C}}$.
]{%
\begin{minipage}{\columnwidth}
    \centering
    \includegraphics[width=0.90\columnwidth]{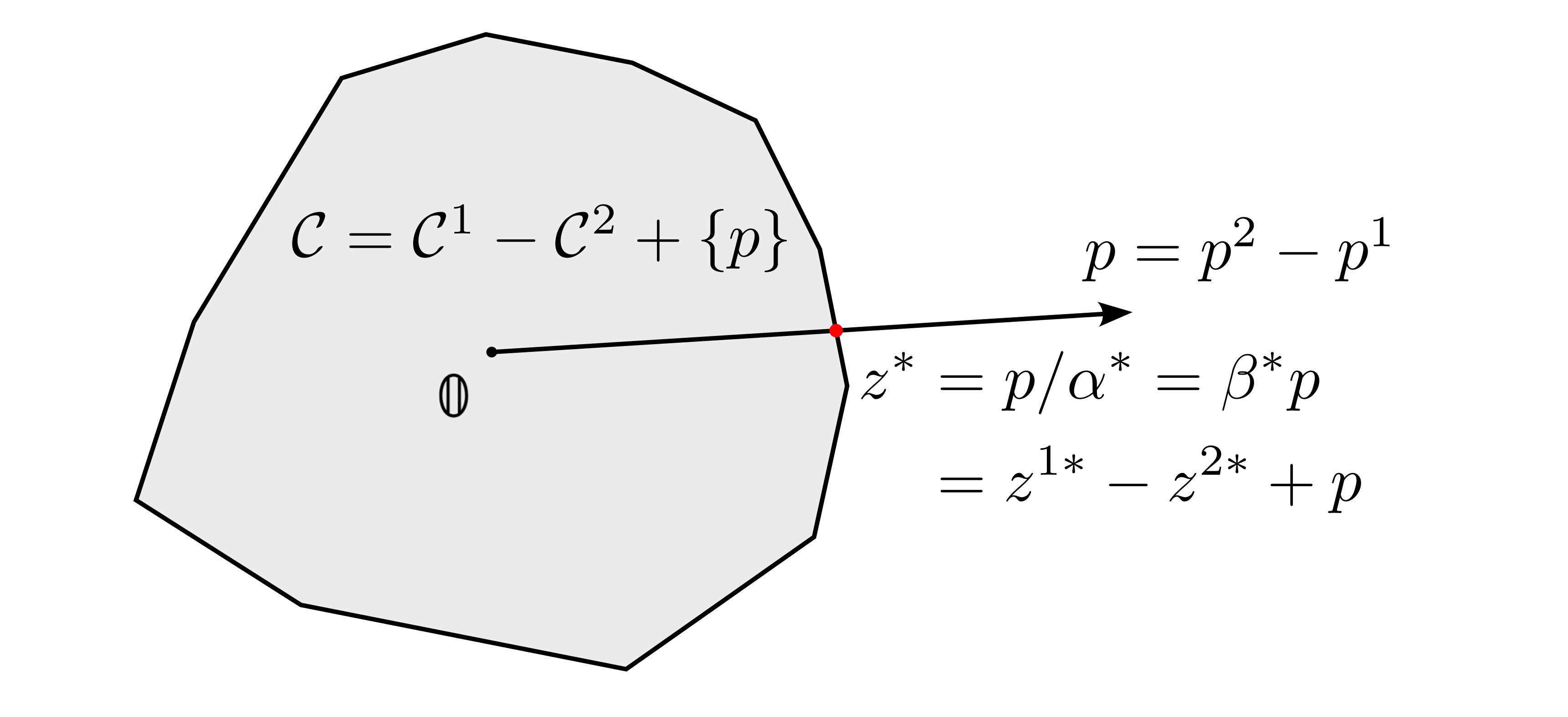}%
    \label{subfig:ray-intersection-problem}
\end{minipage}%
}
\caption{
Example of growth distance computation using \eqref{eq:growth-distance-definition} (in Fig.~\ref{subfig:growth-distance-problem}) and ray intersection computation using \eqref{eq:ray-intersection-definition} (in Fig.~\ref{subfig:ray-intersection-problem}).
By Prop.~\ref{prop:ray-intersection-problem}, the two problems are equivalent.
}
\label{fig:growth-distance-problem}
\end{figure}


In the next section, we will describe an algorithm to solve the ray intersection problem \eqref{eq:ray-intersection-definition}.
Additionally, we will extract an $\epsilon_{tol}$-optimal solution to the growth distance problem \eqref{eq:growth-distance-definition} using an $\epsilon_{tol}$-optimal solution to \eqref{eq:ray-intersection-definition}.

%% file: sections/growth-distance-algorithm.tex
\section{Growth Distance Algorithm}
\label{sec:growth-distance-algorithm}

We now discuss a growth distance algorithm to solve Problem~\ref{problem:growth-distance}, by computing an $\epsilon_{tol}$-optimal solution of \eqref{eq:ray-intersection-definition}.
We assume that the conditions of Problem~\ref{problem:growth-distance} hold, and that $p^1 \neq p^2$ (since if $p^1 = p^2$, $\alpha^* = 0$).
We start by discussing the algorithm's idea.
The main notations used in this section are tabulated in Tab.~\ref{tab:growth-distance-algorithm-notation}.

\subsection{Inner and Outer Polyhedral Approximations}
\label{subsec:inner-outer-polyhedral-approximations}

\begin{table*}[!t]
    \footnotesize
    \caption{Growth Distance Algorithm Notation}
    \label{tab:growth-distance-algorithm-notation}
    \begin{tabularx}{\textwidth}{>{\raggedright\arraybackslash}p{2.2cm}>{\raggedright\arraybackslash}p{1.15cm}L}
    \toprule
    Notation & Reference & Definition \\
    \midrule
    $\mathcal{C}$, $p$ & Prop.~\ref{prop:ray-intersection-problem} & Minkowski difference set, $\mathcal{C} = \mathcal{C}^1 - \mathcal{C}^2 + \{p^2 - p^1\}$, and center point difference, $p = p^2 - p^1$. \\
    \addlinespace[5pt]
    $s_v[\mathcal{C}]$, $s_p[\mathcal{C}]$ & \eqref{subeq:support-function}, \eqref{eq:support-point-function} & Support function and support point function, with $s_v[\mathcal{C}](\lambda) = \langle \lambda, s_p[\mathcal{C}](\lambda)\rangle$. \\
    \addlinespace[5pt]
    $\alpha^*$, $\beta^*$ & \eqref{eq:growth-distance-definition}, \eqref{eq:ray-intersection-definition} & Growth distance and ray intersection value, $\beta^* = 1/\alpha^*$. \\
    \addlinespace[5pt]
    $\lambda_k$, $v_k$, $z_k$ & Sec.~\ref{subsec:inner-outer-polyhedral-approximations} & Normal vector, support function value ($v_k = s_v[\mathcal{C}](\lambda_{k-1})$), and support point ($z_k = s_p[\mathcal{C}](\lambda_{k-1})$) at iteration $k$. \\
    \addlinespace[5pt]
    $\mathcal{P}^i_k$, $\mathcal{P}^o_k$ & \eqref{eq:polyhedral-approximation-definition} & Inner and outer polyhedral approximations of $\mathcal{C}$ at iteration $k$. \\
    \addlinespace[5pt]
    $M^i_k$, $M^o_k$ & \eqref{eq:polyhedral-approximation-definition} & Iteration index sets for the inner and outer polyhedral approximations at iteration $k$. \\
    \addlinespace[5pt]
    $\beta^l_k$, $\beta^u_k$ & \eqref{eq:ray-intersection-lower-bound-expanded-iteration-k}, \eqref{eq:ray-intersection-upper-bound-iteration-k} & Lower and upper ray intersection bounds for $\beta^*$ at iteration $k$. \\
    \addlinespace[5pt]
    $M^{i*}_k$, $(\nu^*_m)_{m \in M^i_k}$ & \eqref{eq:ray-intersection-lower-bound-iteration-k} & Optimal Simplex basis and conic combination coefficients for the inner approximation $\mathcal{P}^i_k$ at iteration $k$. \\
    \bottomrule
\end{tabularx}
\end{table*}

The algorithm proposed in this paper is based on the following observation:
Let $\mathcal{C}^i \subset \mathcal{C} \subset \mathcal{C}^o$ be closed convex sets, where $\mathcal{C}^i$ and $\mathcal{C}^o$ are inner and outer approximations of $\mathcal{C}$.
Also, let $\beta^{l}$, $\beta^*$, and $\beta^{u}$ be the optimal values of the ray intersection problem \eqref{eq:ray-intersection-definition} (for a fixed value of $p$) for $\mathcal{C}^i$, $\mathcal{C}$, and $\mathcal{C}^o$, respectively\footnote{
If \eqref{eq:ray-intersection-definition} is infeasible for $\mathcal{C}^i$, then $\beta^l = -\infty$.
Likewise, if \eqref{eq:ray-intersection-definition} is unbounded for $\mathcal{C}^o$, then $\beta^u = \infty$.
}.
Then, $\beta^{l} \leq \beta^* \leq \beta^u$, i.e., $\beta^l$ and $\beta^u$ provide lower and upper bounds for the optimal value $\beta^*$.
Thus, by appropriately choosing the inner and outer approximations $\mathcal{C}^i$ and $\mathcal{C}^o$, we can closely approximate $\beta^*$.

At any iteration $k \geq 0$ of the algorithm, we use polyhedral inner and outer approximations, denoted by $\mathcal{P}^i_k$ and $\mathcal{P}^o_k$, respectively.
In particular, $\mathcal{P}^i_k$ is represented in vertex form and $\mathcal{P}^o_k$ in halfspace form.

To construct $\mathcal{P}^i_k$ and $\mathcal{P}^o_k$, we first note that for any nonzero normal vector $\lambda$, we have that $\langle \lambda, z\rangle \leq s_v[\mathcal{C}](\lambda)$ $\forall z \in \mathcal{C}$ (by \eqref{subeq:support-function}).
Then, at each iteration $k \geq 1$ of the algorithm, we use the nonzero normal vector $\lambda_{k-1}$ from the previous iteration to compute the support function value $v_k := s_v[\mathcal{C}](\lambda_{k-1}) > 0$ and support point $z_k := s_p[\mathcal{C}](\lambda_{k-1}) \in \mathcal{C}$ using \eqref{eq:minkowski-difference-support-function}.
The polyhedral approximations $\mathcal{P}^i_k$ and $\mathcal{P}^o_k$ at iteration $k$ are then constructed as follows:
\begin{subequations}
\label{eq:polyhedral-approximation-definition}
\begin{align}
    \mathcal{P}^i_k & = \conv\{z_m: m \in M^i_k\}, \label{subeq:polyhedral-approximation-definition-inner} \\
    \mathcal{P}^o_k & = \{z: \langle \lambda_{m-1}, z\rangle \leq v_m, \, \forall m \in M^o_k\}, \label{subeq:polyhedral-approximation-definition-outer}
\end{align}
\end{subequations}
where $M^i_k \subset \{-l+1, ..., k\}$ and $M^o_k \subset \{1, ..., k\}$ are the subsets of iteration indices chosen for the construction of the approximations (the initialization of the index sets will be discussed in the following subsection).
Subsequently, the lower and upper bounds $\beta^l_k$ and $\beta^u_k$ (the optimal values of \eqref{eq:ray-intersection-definition} for $\mathcal{P}^i_k$ and $\mathcal{P}^o_k$, respectively) are computed to check convergence.
Finally, the normal vector $\lambda_k$ is computed for the next iteration.

Thus, it follows that $\mathcal{P}^i_k \subset \mathcal{C} \subset \mathcal{P}^o_k$, and therefore $\beta^l_k \leq \beta^* \leq \beta^u_k$ for all iterations $k \geq 1$ (and by construction at $k = 0$).
$\beta^u_k$ can be computed using \eqref{eq:ray-intersection-definition} and \eqref{subeq:polyhedral-approximation-definition-outer} as
\begin{subequations}
\label{eq:ray-intersection-upper-bound-iteration-k}
\begin{alignat}{4}
    \beta^u_k = \ && \max_{\beta, z} \ & \beta, \\
    && \text{s.t.} \ & \langle \lambda_{m-1}, z\rangle \leq v_m, \quad \forall m \in M^o_k, \\
    &&& z = \beta p. \\
    = \ \min_{\substack{m \in M^o_k\\ \langle \lambda_{m-1}, p\rangle > 0}} \ \frac{v_m}{\langle \lambda_{m-1}, p\rangle}. \span\span\span
\end{alignat}
\end{subequations}
The above equation suggests we can simplify the outer approximation $\mathcal{P}^o_k$ by choosing $M^o_k$ as a singleton set.
Computing $\beta^l_k$ requires the solution of the following linear program (LP), obtained from \eqref{eq:ray-intersection-definition} and \eqref{subeq:polyhedral-approximation-definition-inner}:
\begin{subequations}
\label{eq:ray-intersection-lower-bound-expanded-iteration-k}
\begin{align}
\beta^l_k = \ \max_{\beta, (\mu_m)_{m \in M^i_k} \geq \mathbb{0}} \ \ & \beta, \\[-5pt]
\text{s.t.} \
& \textstyle\sum\limits_{m \in M^i_k} \mu_m \Bigl[\begin{matrix}z_m \\ 1 \end{matrix}\Bigr] = \Bigl[\begin{matrix}\beta p \\ 1\end{matrix}\Bigr], \label{subeq:ray-intersection-lower-bound-expanded-iteration-k-equality}
\end{align}
\end{subequations}
where the variables $(\mu_m)_{m \in M^i_k}$ are the convex combination coefficients for the vertices of $\mathcal{P}^i_k$.
For all iterations $k \geq 0$, we will ensure that \eqref{eq:ray-intersection-lower-bound-expanded-iteration-k} is feasible and $\beta^l_k > 0$.
Then, substituting $\nu_m = \mu_m / \beta$ in \eqref{eq:ray-intersection-lower-bound-expanded-iteration-k} (this can be done because $\beta^l_k > 0$), we obtain the following equivalent LP: 
\begin{subequations}
\label{eq:ray-intersection-lower-bound-iteration-k}
\begin{align}
(\beta^l_k)^{-1} = \ \min_{(\nu_m)_{m \in M^i_k} \geq \mathbb{0}} \ \ & \textstyle\sum\limits_{m \in M^i_k} \nu_m, \label{subeq:ray-intersection-lower-bound-iteration-k-cost} \\
\text{s.t.} \
& \textstyle\sum\limits_{m \in M^i_k} \nu_m z_m = p, \label{subeq:ray-intersection-lower-bound-iteration-k-equality} 
\end{align}
\end{subequations}
where the variables $(\nu_m)_{m \in M^i_k}$ are the conic combination coefficients for the vertices of $\mathcal{P}^i_k$.
The LP \eqref{eq:ray-intersection-lower-bound-iteration-k} is in standard form, and can be solved using the Simplex method~\cite[Ch.~4]{nemirovski2024introduction}.
The Simplex method returns an optimal basis $M^{i*}_k \subset M^i_k$ with $|M^{i*}_k| \leq l$ such that the corresponding optimal solution $(\nu^*_m)_{m \in M^i_k}$ to \eqref{eq:ray-intersection-lower-bound-iteration-k} satisfies $\nu^*_m = 0$ $\forall m \in M^i_k \setminus M^{i*}_k$ (the existence of such a basis also follows from the Caratheodory Theorem~\cite[Thm.~2.4]{nemirovski2024introduction} applied to \eqref{eq:ray-intersection-lower-bound-expanded-iteration-k}).
Fig.~\ref{subfig:algorithm-iteration-k} depicts an example of the algorithm state at the end of iteration $k$. 

Next, we will discuss the initialization and update steps for the polyhedral approximations $\mathcal{P}^i_k$ and $\mathcal{P}^o_k$, the index sets $M^i_k$ and $M^o_k$, and the normal vector $\lambda_k$.

\begin{figure}[!t]
\centering
\subfloat[
The algorithm state at the end of iteration $k$.
The polyhedra $\mathcal{P}^i_k$ (in dark gray) and $\mathcal{P}^o_k$ (in light gray) approximate the convex set $\mathcal{C}$ (in blue).
For this example, the index sets used to construct $\mathcal{P}^i_k$ and $\mathcal{P}^o_k$ are given by $M^i_k = \{m_1, m_2, m_3\}$ and $M^o_k = \{m_2\}$.
The bounds $\beta^l_k$ and $\beta^u_k$ on the optimal value $\beta^*$ are computed using the intersection points of the ray $\vec{\mathbb{0}p}$ (in red) with $\mathcal{P}^i_k$ and $\mathcal{P}^o_k$, respectively.
The normal vector $\lambda_{m_2-1}$ achieves the value $\beta^u_k$ for \eqref{eq:ray-intersection-upper-bound-iteration-k}, while the optimal basis for the LP \eqref{eq:ray-intersection-lower-bound-iteration-k} that achieves the value $\beta^l_k$ is $M^{i*}_k = \{m_1, m_2\}$.
]{%
\begin{minipage}{\columnwidth}
    \centering
    \includegraphics[width=0.90\columnwidth]{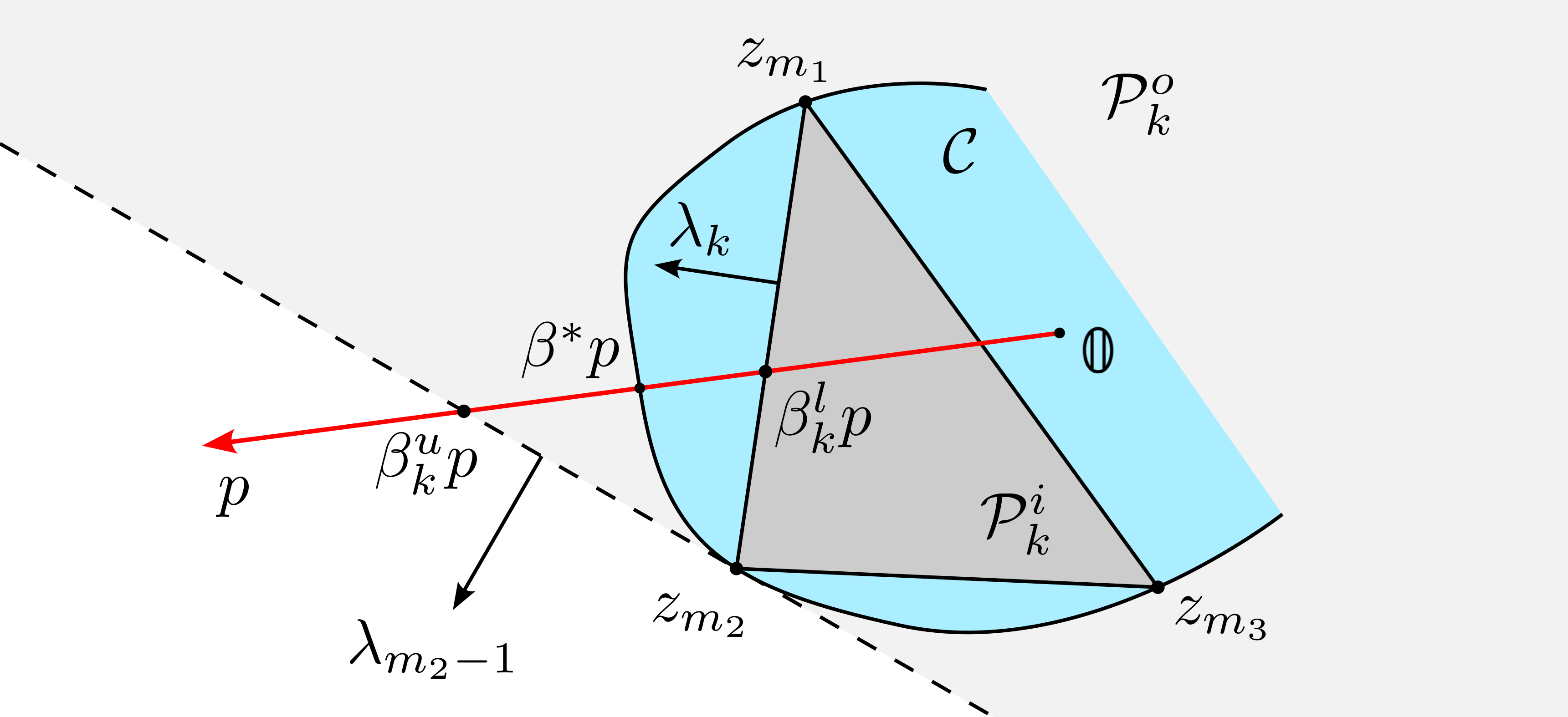}%
    \label{subfig:algorithm-iteration-k}
\end{minipage}%
}
\hfil
\subfloat[
The algorithm during iteration $k+1$, after the \emph{outer approximation update} step from Fig.~\ref{subfig:algorithm-iteration-k}.
The index set $M^o_{k+1}$, corresponding to the outer approximation $\mathcal{P}^o_{k+1}$, is set to $\{k+1\}$ according to \eqref{eq:outer-approximation-update-iteration-k}.
We can see that $\beta^l_k \leq \beta^* \leq \beta^u_{k+1} \leq \beta^u_k$.
]{%
\begin{minipage}{\columnwidth}
    \centering
    \includegraphics[width=0.90\columnwidth]{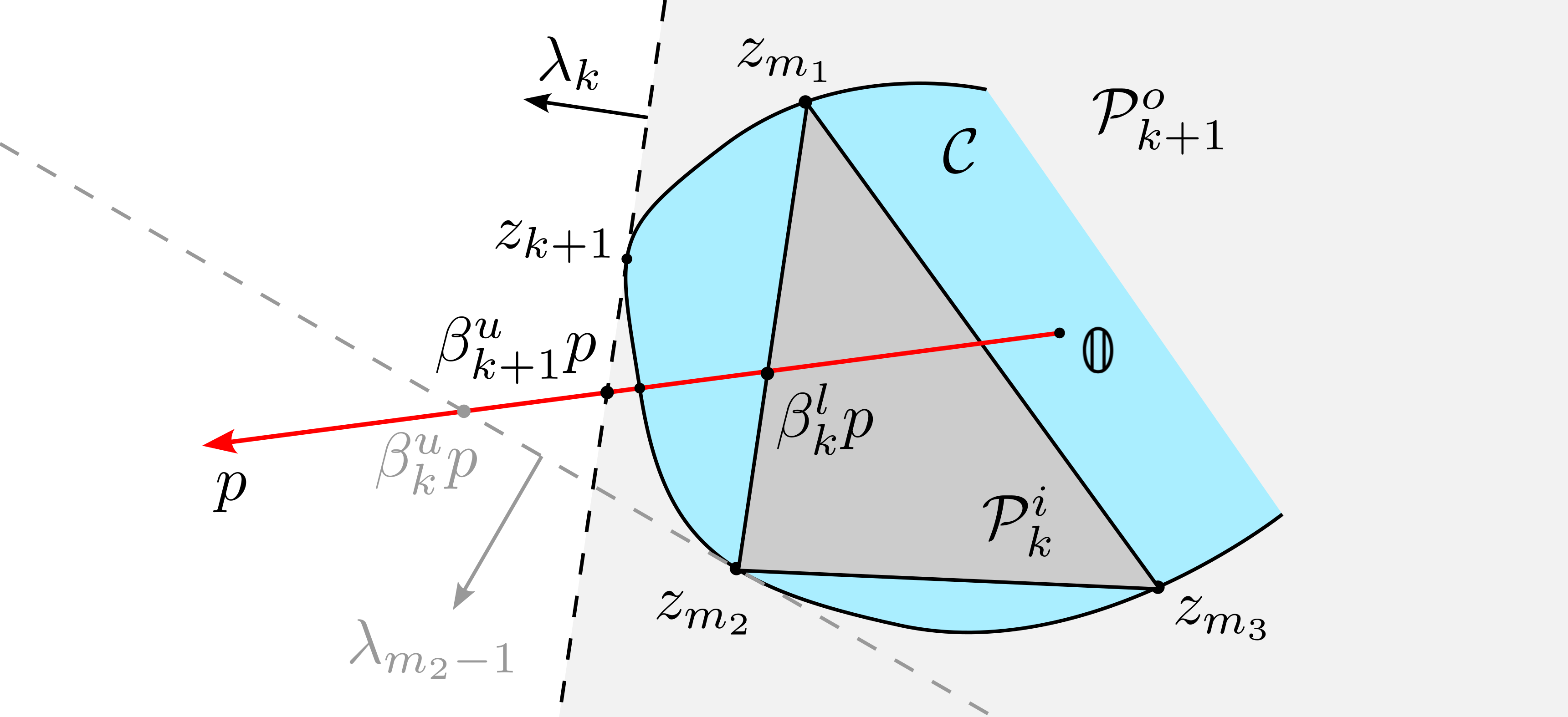}%
    \label{subfig:algorithm-iteration-outer-k-plus-1}
\end{minipage}%
}
\hfil
\subfloat[
The algorithm at the end of iteration $k+1$, after the \emph{inner approximation update} step from Fig.~\ref{subfig:algorithm-iteration-outer-k-plus-1}.
The support point $z_{k+1}$, obtained using $\lambda_k$, is added to the set of inner approximation vertices.
The lower bound $\beta^l_{k+1}$ and optimal basis $M^{i*}_{k+1} = \{m_2, k+1\}$ are computed using the Simplex method (see Sec.~\ref{subsubsec:inner-approximation-update}).
By \eqref{eq:simplex-method-iteration-k-normal}, $\lambda_{k+1}$ is orthogonal to the unique affine space containing $\{z_m: m \in M^{i*}_{k+1}\}$ with $\langle \lambda_{k+1}, p\rangle > 0$.
We can see that $\beta^l_k \leq \beta^l_{k+1} \leq \beta^* \leq \beta^u_{k+1}$.
A pruning step is used to reduce the number of inner approximation vertices (see Rem.~\ref{rem:inner-approximation-pruning}).
]{%
\begin{minipage}{\columnwidth}
    \centering
    \includegraphics[width=0.90\columnwidth]{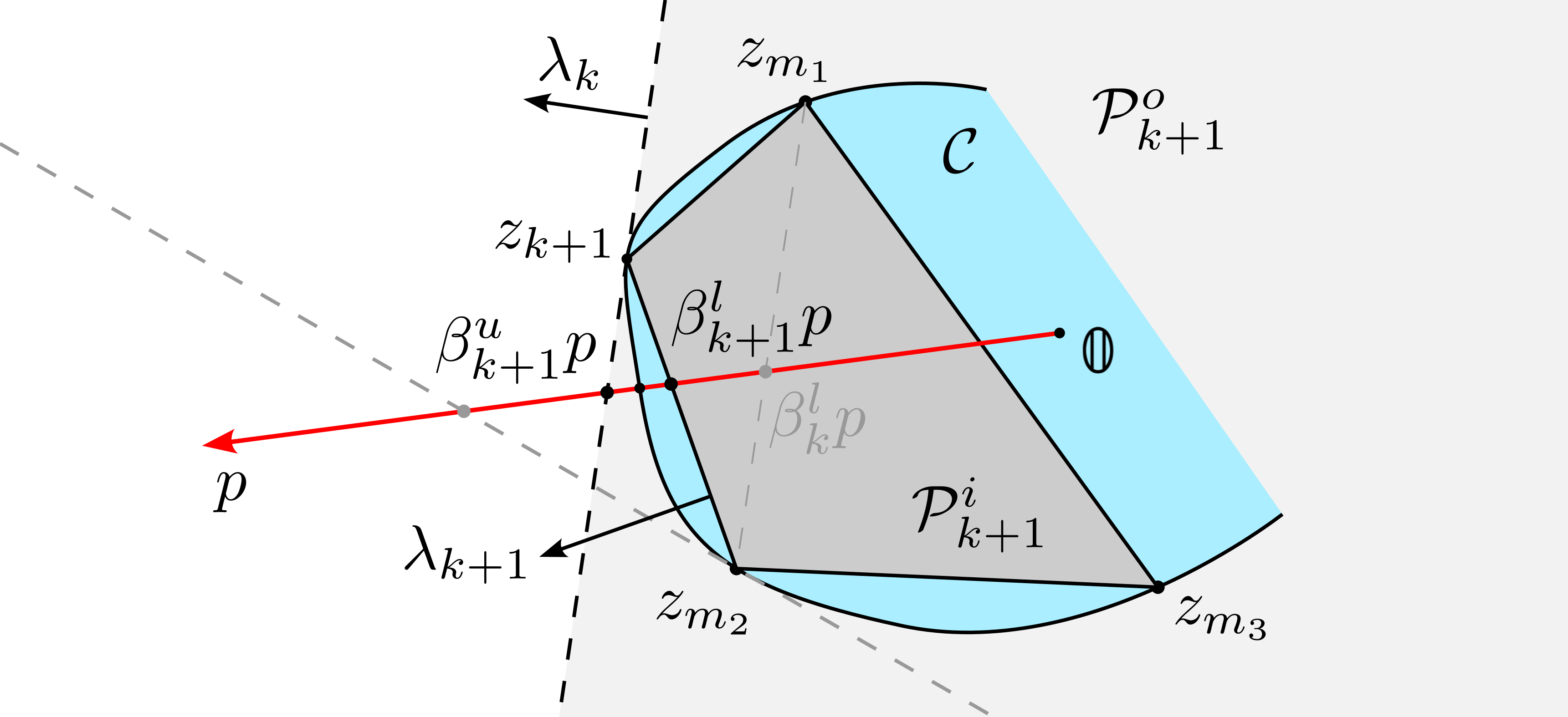}%
    \label{subfig:algorithm-iteration-inner-k-plus-1}
\end{minipage}%
}
\caption{
Examples of the growth distance algorithm state at the end of iteration $k$, and after the inner and outer approximation update steps.
The inner and outer approximations $\mathcal{P}^i_k$ and $\mathcal{P}^o_k$ are updated to better approximate the ray intersection point $\beta^* p$ and obtain an $\epsilon_{tol}$-optimal solution to the ray intersection problem \eqref{eq:ray-intersection-definition}.
Lastly, an $\epsilon_{tol}$-optimal solution to the growth distance problem \eqref{eq:growth-distance-definition} is retrieved from the $\epsilon_{tol}$-optimal solution to the ray intersection problem (see Prop.~\ref{prop:ray-intersection-problem}).
Note that $p$ is constant during the algorithm execution.
}
\label{fig:growth-distance-algorithm}
\end{figure}

\subsection{Initialization}
\label{subsec:initialization}

\subsubsection{Outer Approximation Initialization}
\label{subsubsec:outer-approximation-initialization}

At iteration $k = 0$, we choose the index set $M^o_0 := \emptyset$.
Using \eqref{subeq:polyhedral-approximation-definition-outer} and \eqref{eq:ray-intersection-upper-bound-iteration-k}, we get that the outer approximation $\mathcal{P}^o_0 = \real^l$ and the ray intersection upper bound $\beta^u_0 = \infty$.

\subsubsection{Inner Approximation Initialization}
\label{subsubsec:inner-polyhedral-initialization}

To initialize the inner approximation $\mathcal{P}^i_0 \subset \mathcal{C}$, we choose $l$ points $z_m \in \mathcal{C}$ for $m \in M^i_0 := \{-l+1, ..., 0\}$ such that
\begin{equation}
\label{eq:inner-approximation-initialization-properties}
\begin{gathered}
    \{z_m\}_{m \in M^i_0} \subset \mathcal{C} \text{ is a linearly independent set}, \\
    p \in \cone\{z_m: m \in M^i_0\}.
\end{gathered}
\end{equation}
Before elaborating on the rationale, we first show how to select such points.
Let, for $i \in \{1, 2\}$, the inradius of the set $\mathcal{C}^i$ at the center point $p^i$, $r^i = r(\mathcal{C}^i, p^i)$, be known.
So, we have that $B_{\underline{r}}(\mathbb{0}) \subset \mathcal{C} = \mathcal{C}^1 - \mathcal{C}^2 + \{p\}$, where $\underline{r} = r^1 + r^2 > 0$.
Then, the points $(z_m)_{m \in M^i_0}$ can be chosen by first computing $l-1$ unit vectors $(p^\perp_m)_{m \in \{-l+1, ..., -1\}}$ that are independent and orthogonal to $p$, and setting
\begin{subequations}
\label{eq:inner-approximation-initialization-construction}
\begin{align}
    z_m & = K p^\perp_m + \epsilon p/\lVert p\rVert, \quad \forall m \in \{-l+1, ..., -1\}, \\
    z_0 & = l\epsilon p/\lVert p\rVert - \textstyle\sum\limits_{m= -l+1}^{-1} z_m,
\end{align}
\end{subequations}
where $\epsilon > 0$ is a small constant and $K > 0$ is such that $\lVert z_m\rVert \leq \underline{r}$ for all $m \in M^i_0$.
Note that $\sum_{m \in M^i_0} z_m = l \epsilon p/\lVert p\rVert$, and thus $(z_m)_{m \in M^i_0}$ satisfy \eqref{eq:inner-approximation-initialization-properties}.

Such a choice of $\mathcal{P}^i_0$ is made to satisfy two key properties.
First, we want to ensure \textit{primal feasibility} at $k = 0$, i.e., there should exist $(\nu_m)_{m \in M^i_k} \geq \mathbb{0}$ such that \eqref{subeq:ray-intersection-lower-bound-iteration-k-equality} holds.
Then, we can see that $(\beta, z)$, where
\begin{equation}
\label{eq:conic-combination-to-optimal-solution}
    \beta = \biggl(\textstyle\sum\limits_{m \in M^i_k} \nu_m\biggr)^{-1}, \quad z = \textstyle\sum\limits_{m \in M^i_k} \beta \nu_m z_m,
\end{equation}
is feasible for \eqref{eq:ray-intersection-definition}.
Second, we want the \textit{linear independence} of the optimal points $\{z_m\}_{m \in M^{i*}_k}$ at $k = 0$, to initialize the Simplex method for \eqref{eq:ray-intersection-lower-bound-iteration-k} at iteration $k+1$.
Since $\{z_m\}_{m \in M^i_0}$ is a linearly independent set and $M^{i*}_0 = M^i_0$, the linear independence property is satisfied at $k = 0$.

Finally, the normal vector $\lambda_k$ at $k = 0$ is chosen to be orthogonal to the affine space containing $\{z_m\}_{m \in M^{i*}_k}$ with $\langle \lambda_k, p\rangle > 0$.
Following the construction in \eqref{eq:inner-approximation-initialization-construction}, we can choose the normal vector $\lambda_0 = p/\lVert p\rVert_\infty$.
The initialization subroutine for the algorithm is outlined in Alg.~\ref{alg:initialization}.

\algsetup{indent=1em} 
\begin{algorithm}[!t]
\caption{Initialization}\label{alg:initialization}
\begin{algorithmic}[1] 
\REQUIRE $p$, $\underline{r}$ \LONGCOMMENT{150pt}{Relative center position and inradius lower bound, $\underline{r} = r^1 + r^2$}
\STATE \textbf{Function} \texttt{Initialize}():
\STATE \hspace{1em} $\displaystyle(M^o_0, \beta^u_0) \gets (\emptyset, \infty)$
\STATE \hspace{1em} $\displaystyle M^i_0 \gets \{-l+1, ..., 0\}$
\STATE \hspace{1em} $(z_m)_{m \in M^i_0} \gets$ computed using \eqref{eq:inner-approximation-initialization-construction}
\STATE \hspace{1em} $\displaystyle (M^{i*}_0, \beta^l_0) \gets (M^i_0, \epsilon/\lVert p\rVert)$
\STATE \hspace{1em} $\displaystyle\lambda_0 \gets p/\lVert p\rVert_\infty$
\STATE \hspace{1em} \textbf{return} $\displaystyle (M^o_0, \beta^u_0), (M^i_0, M^{i*}_0, \beta^l_0), \lambda_0$
\STATE \textbf{End Function}
\end{algorithmic}
\end{algorithm}

Next, we discuss the update step for the inner and outer approximations, ensuring that the primal feasibility and linear independence properties of the inner approximation are satisfied at each iteration $k \geq 1$.
The importance of the primal feasibility property comes from the fact that the algorithm always maintains a feasible solution $(\beta, z)$ to \eqref{eq:ray-intersection-definition}.
Thus, when the algorithm optimally terminates, we can compute an $\epsilon_{tol}$-optimal solution to \eqref{eq:ray-intersection-definition} using \eqref{eq:conic-combination-to-optimal-solution}.
The linear independence property is crucial because it allows us to initialize the Simplex method for the LP \eqref{eq:ray-intersection-lower-bound-iteration-k} at the next iteration, as emphasized in the following subsection.

\subsection{Polyhedral Approximation Update}
\label{subsec:polyhedral-approximation-update}

At the start of iteration $k \geq 1$, the normal vector $\lambda_{k-1}$ from the previous iteration is used to compute the support function value $v_k = s_v[\mathcal{C}](\lambda_{k-1}) > 0$ and support point $z_k = s_p[\mathcal{C}](\lambda_{k-1}) \in \mathcal{C}$ using \eqref{eq:minkowski-difference-support-function}.
We assume that the primal feasibility and linear independence properties are satisfied at the end of iteration $k-1$, i.e., the optimal basis $M^{i*}_{k-1} \subset M^i_{k-1}$ is such that $|M^{i*}_{k-1}| = l$ and
\begin{equation}
\label{eq:inner-approximation-properties-iteration-k-1}
\begin{gathered}
    \{z_m\}_{m \in M^{i*}_{k-1}} \subset \mathcal{C} \text{ is a linearly independent set}, \\
    p \in \cone\{z_m: m \in M^{i*}_{k-1}\}.
\end{gathered}
\end{equation}
Additionally, we assume that $|M^o_{k-1}| = 1$ for $k \geq 2$ and the normal vector $\lambda_{k-1}$ satisfies
\begin{equation}
\label{eq:normal-properties-iteration-k-1}
\begin{gathered}
    \langle \lambda_{k-1}, z_{m_1}\rangle = \langle \lambda_{k-1}, z_{m_2}\rangle, \quad \forall m_1, m_2 \in M^{i*}_{k-1}, \\
    \langle \lambda_{k-1}, p\rangle > 0.
\end{gathered}
\end{equation}
The first condition in \eqref{eq:normal-properties-iteration-k-1} implies that $\lambda_{k-1}$ is orthogonal to the unique affine space containing $\{z_m\}_{m \in M^{i*}_{k-1}}$.
From the construction in Sec.~\ref{subsec:initialization}, \eqref{eq:inner-approximation-properties-iteration-k-1} and \eqref{eq:normal-properties-iteration-k-1} are satisfied at the start of iteration $k = 1$.
Assuming \eqref{eq:inner-approximation-properties-iteration-k-1} and \eqref{eq:normal-properties-iteration-k-1} hold at the start of iteration $k$, we will show in Sec.~\ref{subsubsec:inner-approximation-update} that they also hold at the end of iteration $k$, thus showing that \eqref{eq:inner-approximation-properties-iteration-k-1} and \eqref{eq:normal-properties-iteration-k-1} are algorithmic invariants.
We also note that a normal vector $\lambda_{k-1}$ satisfying the first condition in \eqref{eq:normal-properties-iteration-k-1} can be obtained by solving the set of equations $\langle \lambda_{k-1}, z_{m}\rangle = 1$ for all $m \in M^{i*}_{k-1}$, which results in
\begin{equation}
\label{eq:normal-vector-iteration-k-1}
    \lambda_{k-1} = \bigl[\begin{matrix}z_{m_1} & \cdots & z_{m_l} \end{matrix}\bigr]^{-\top} \mathbb{1},
\end{equation}
where $\{m_1, ..., m_l\} = M^{i*}_{k-1}$.
Note the matrix in \eqref{eq:normal-vector-iteration-k-1} is nonsingular, since $\{z_m\}_{m \in M^{i*}_{k-1}}$ is a linearly independent set (see \eqref{eq:inner-approximation-properties-iteration-k-1}).
Additionally, the value of $\lambda_{k-1}$ in \eqref{eq:normal-vector-iteration-k-1} satisfies the second condition in \eqref{eq:normal-properties-iteration-k-1} since $\langle \lambda_{k-1}, p\rangle = (\beta^l_{k-1})^{-1} > 0$ (by \eqref{eq:ray-intersection-lower-bound-iteration-k}).
We now detail how $v_k$ and $z_k$ are used to update the polyhedral approximations $\mathcal{P}^i_k$ and $\mathcal{P}^o_k$.

\subsubsection{Outer Approximation Update}
\label{subsubsec:outer-approximation-update}

By assumption, $M^o_{k-1}$ is the empty set at $k = 1$ and a singleton set for $k \geq 2$.
Since we want to find the best upper bound $\beta^u_k$ at iteration $k$, we compute $\beta^u_k$ and $M^o_k$ using \eqref{eq:ray-intersection-upper-bound-iteration-k} as follows:
\begin{subequations}
\label{eq:outer-approximation-update-iteration-k}
\begin{align}
    M^o_k & = \{m^*\}, \quad m^* \in \argmin_{m \in M^o_{k-1} \cup \{k\}} \frac{v_m}{\langle\lambda_{m-1}, p\rangle}, \\
    \beta^u_k & = \frac{v_{m^*}}{\langle\lambda_{m^*-1}, p\rangle} = \min\biggl\{\beta^u_{k-1}, \frac{v_k}{\langle\lambda_{k-1}, p\rangle}\biggr\}.
\end{align}
\end{subequations}
In other words, $\beta^u_k$ keeps track of the minimum value of $v_k/\langle\lambda_{k-1}, p\rangle$, and $M^o_k$ keeps track of the index that achieves the minimum value.
Note that the constraint $\langle \lambda_{m-1}, p\rangle > 0$ in \eqref{eq:ray-intersection-upper-bound-iteration-k} is not required since it is always satisfied by \eqref{eq:normal-properties-iteration-k-1}.

For the algorithm state at iteration $k$ in Fig.~\ref{subfig:algorithm-iteration-k}, Fig.~\ref{subfig:algorithm-iteration-outer-k-plus-1} depicts how the outer approximation is updated.

\subsubsection{Inner Approximation Update}
\label{subsubsec:inner-approximation-update}

The inner approximation $\mathcal{P}^i_k$ and lower bound $\beta^l_k$ are updated by solving the LP subproblem \eqref{eq:ray-intersection-lower-bound-iteration-k} with the new point $z_k$.
In particular, we set $M^i_k = M^i_{k-1} \cup \{k\}$ and solve the LP \eqref{eq:ray-intersection-lower-bound-iteration-k} using the Simplex method \cite[Ch.~4]{nemirovski2024introduction} with the initial basis as the optimal basis $M^{i*}_{k-1}$ from the previous iteration.
An added benefit of using the Simplex method is that the normal vector $\lambda_k$ can be computed together with the LP solution.
Next, we elaborate on the computations involved in one iteration \textit{of the Simplex method}.
Our goal is not to describe the Simplex method in its entirety, but to highlight that both \eqref{eq:inner-approximation-properties-iteration-k-1} and \eqref{eq:normal-properties-iteration-k-1} hold when the Simplex method for the LP \eqref{eq:ray-intersection-lower-bound-iteration-k} terminates.

The Simplex method (for solving a standard form LP) iteratively chooses basic variables to reduce the LP cost, while maintaining primal feasibility.
For the standard form LP \eqref{eq:ray-intersection-lower-bound-iteration-k} at iteration $k$ with initial basis $M^{i*}_k$, the nonbasic variables are the variables $(\nu_m)$ for $m \in M^i_k \setminus M^{i*}_k$ which are set to zero.
The values of the basic variables $(\nu_m)$ for $m \in M^{i*}_k$ are computed to be nonnegative and satisfy \eqref{subeq:ray-intersection-lower-bound-iteration-k-equality}.
Then, the Simplex method updates the basis $M^{i*}_k$ to reduce the cost associated with the basic variables.
So, the Simplex method iteratively computes primal feasible solutions to the LP \eqref{eq:ray-intersection-lower-bound-iteration-k} with monotonically decreasing cost.

Conceptually, the Simplex algorithm is represented by a \emph{tableau}.
We start with the initial basis $M^{i*}_k = M^{i*}_{k-1}$, which is guaranteed to be feasible at iteration $k$, and represent the Simplex method state with the following tableau~\cite[Ch.~4]{nemirovski2024introduction}:
\begin{equation}
\label{eq:simplex-tableau-iteration-k-initial}
\renewcommand{\arraystretch}{1.25}
\begin{array}{ccc:cc|c}
    \multicolumn{3}{c:}{\mathbb{1}^\top} & \mathbb{1}^\top & 1 & 0 \\
    \hline
    z_{m_1} & \cdots & \multicolumn{1}{c:}{z_{m_l}} & \cdots & z_k & p
\end{array}
\end{equation}
where $\{m_1, ..., m_l\} = M^{i*}_k$ is the current basis.
Each column, except the rightmost one, corresponds to a nonnegative variable $\nu_m$ in the LP \eqref{eq:ray-intersection-lower-bound-iteration-k}, the top row corresponds to the cost coefficients in \eqref{subeq:ray-intersection-lower-bound-iteration-k-cost}, and the remaining rows correspond to the equality constraint \eqref{subeq:ray-intersection-lower-bound-iteration-k-equality}.
Similar to \eqref{eq:normal-vector-iteration-k-1}, the normal vector $\lambda_k$ is computed using the current basis $M^{i*}_k$ as
\begin{align}
\label{eq:simplex-method-iteration-k-normal}
    \lambda_k & = Z^{-\top} \mathbb{1}, \quad Z = \bigl[\begin{matrix}z_{m_1} & \cdots & z_{m_l} \end{matrix}\bigr].
\end{align}
Then, the tableau in \eqref{eq:simplex-tableau-iteration-k-initial} is pre-multiplied by the matrix
\begin{equation*}
    \begin{bmatrix}
        1 & -\lambda_k^\top \\
        \mathbb{0} & Z^{-1}
    \end{bmatrix},
\end{equation*}
which results in the following reduced tableau:
\begin{equation}
\label{eq:reduced-simplex-tableau-iteration-k-initial}
\renewcommand{\arraystretch}{1.25}
\begin{array}{c:cc|c}
    \multicolumn{1}{c:}{\mathbb{0}^\top} & d^\top & \tilde{d} & -(\beta^l_k)^{-1} \\
    \hline
    \multicolumn{1}{c:}{I} & \cdots & (\tilde{\nu}_m)_{m \in M^{i*}_k} & (\nu^*_m)_{m \in M^{i*}_k}
\end{array}
\end{equation}
where the terms $-(\beta^l_k)^{-1}$ and $(\nu^*_m)_{m \in M^{i*}_k} \geq \mathbb{0}$ in the rightmost column are the negative cost and basic variable values associated with the current basis $M^{i*}_k$.
Likewise, the term $(\tilde{\nu}_m)_{m \in M^{i*}_k}$ is the coordinate of $z_k$ with respect to the current basis vectors $(z_m)_{m \in M^{i*}_k}$.

The Simplex method proceeds by replacing one index (the outgoing index) in the current basis $M^{i*}_k$ by an index in $M^i_k \setminus M^{i*}_k$ (the incoming index).
The incoming index $m_{in} \in M^i_k \setminus M^{i*}_k$ is chosen as the index with the most negative value in the top row of the reduced tableau \eqref{eq:reduced-simplex-tableau-iteration-k-initial}; if all values in the top row are nonnegative, then the Simplex method reaches optimality.
\textit{For the first iteration of the Simplex method}, the vector $d$ in \eqref{eq:reduced-simplex-tableau-iteration-k-initial} is nonnegative (since $M^{i*}_k = M^{i*}_{k-1}$ is the optimal basis at iteration $k-1$) and $\tilde{d} \leq 0$ since
\begin{equation*}
\begin{split}
    \tilde{d} & = 1 - \langle \lambda_k, z_k\rangle = 1 - \langle \lambda_{k-1}, z_k\rangle = 1 - s_v[\mathcal{C}](\lambda_{k-1}), \\
    & \leq 1 - \langle \lambda_{k-1}, z_{m_1}\rangle = 0.
\end{split}
\end{equation*}
So, if $\tilde{d} = 0$, the Simplex method at iteration $k$ reaches optimality (in fact, the growth distance algorithm also reaches optimality, i.e., $\beta^u_k = \beta^l_k$).
If $\tilde{d} < 0$, the incoming index at the first Simplex iteration is $m_{in} = k$.
The outgoing index $m_{out} \in M^{i*}_k$ is determined as follows (see \cite[Ch.~4.3.1]{nemirovski2024introduction}):
\begin{equation}
\label{eq:simplex-iteration-k-outgoing-index}
    m_{out} = \argmin_{\substack{m \in M^{i*}_k, \tilde{\nu}_m > 0}} \ \frac{\nu^*_m}{\tilde{\nu}_m}. 
\end{equation}
At least one $m \in M^{i*}_k$ such that $\tilde{\nu}_m \geq 1/l > 0$ exists since
\begin{equation*}
\begin{split}
    \textstyle\sum\limits_{m \in M^{i*}_k} \tilde{\nu}_m & = 1 - \tilde{d} \geq 1.
\end{split}
\end{equation*}
After the incoming and outgoing indices are determined, the current basis is updated as $M^{i*}_k \gets (M^{i*}_k \cup \{m_{in}\}) \setminus \{m_{out}\}$, and the Simplex method continues from the updated tableau, similar to \eqref{eq:simplex-tableau-iteration-k-initial}.
Finally, the Simplex method optimally terminates in a finite number of iterations (Bland's rule can be applied to prevent cycling~\cite[Ch.~4.3.3]{nemirovski2024introduction}).

We highlight a key property of the Simplex method applied to \eqref{eq:ray-intersection-lower-bound-iteration-k} \textit{at iteration $k \geq 1$}.
The Simplex method ensures the basis $M^{i*}_k$ maintains primal feasibility and linear independence (as specified in \eqref{eq:inner-approximation-properties-iteration-k-1}) throughout its execution and at optimality~\cite[Ch.~4.3.1]{nemirovski2024introduction}.
Similarly, the normal vector $\lambda_k$ (computed per \eqref{eq:simplex-method-iteration-k-normal}) satisfies \eqref{eq:normal-properties-iteration-k-1} for its corresponding basis throughout the Simplex iterations, including at optimality.
Therefore, both \eqref{eq:inner-approximation-properties-iteration-k-1} and \eqref{eq:normal-properties-iteration-k-1} hold at the start of all iterations $k \geq 0$, and are algorithmic invariants.

\algsetup{indent=1em} 
\begin{algorithm}[!t]
\caption{Polyhedral Approximation Update}\label{alg:update}
\begin{algorithmic}[1] 
\REQUIRE $p$, $s_v[\mathcal{C}]$, $s_p[\mathcal{C}]$ \LONGCOMMENT{120pt}{Relative center position and support functions}
\STATE \textbf{Function} \texttt{Update}$\displaystyle(M^o_{k-1}, (M^i_{k-1}, M^{i*}_{k-1}), \lambda_{k-1})$:
\STATE \hspace{1em} $\displaystyle(v_k, z_k) \gets (s_v[\mathcal{C}](\lambda_{k-1}), s_p[\mathcal{C}](\lambda_{k-1}))$ \COMMENT{see \eqref{eq:minkowski-difference-support-function}}
\STATE \hspace{1em} $\displaystyle(M^o_k, \beta^u_k) \gets$ computed using \eqref{eq:outer-approximation-update-iteration-k}
\STATE \hspace{1em} $\displaystyle(M^i_k, M^{i*}_k) \gets (M^i_{k-1} \cup \{k\}, M^{i*}_{k-1})$
\STATE \hspace{1em} $\displaystyle(M^{i*}_k, \beta^l_k) \gets$ \LONGTEXT{158pt}{computed using the Simplex method (see Sec.~\ref{subsubsec:inner-approximation-update})}
\STATE \hspace{1em} $\displaystyle M^i_k \gets \texttt{Prune}(M^i_k, M^{i*}_k)$ \COMMENT{see Rem.~\ref{rem:inner-approximation-pruning}} \label{subalg:update-pruning}
\STATE \hspace{1em} $\displaystyle\lambda_k \gets$ computed using \eqref{eq:simplex-method-iteration-k-normal}
\STATE \hspace{1em} \textbf{return} $\displaystyle(M^o_k, \beta^u_k), (M^i_k, M^{i*}_k, \beta^l_k), \lambda_k$
\STATE \textbf{End Function}
\end{algorithmic}
\end{algorithm}

For the algorithm state at iteration $k$ in Fig.~\ref{subfig:algorithm-iteration-k}, Fig.~\ref{subfig:algorithm-iteration-inner-k-plus-1} depicts how the inner approximation is updated.
The update subroutine for the algorithm is outlined in Alg.~\ref{alg:update}.
An additional step to reduce the size of $M^i_k$, and thus the size of the LP \eqref{eq:ray-intersection-lower-bound-iteration-k}, is used for the inner approximation update (see Alg.~\ref{alg:update}, line~\ref{subalg:update-pruning}), as discussed in the following remark.

\begin{remark} (Index set pruning for the inner approximation)
\label{rem:inner-approximation-pruning}
At each iteration $k \geq 1$ of the algorithm, the index $k$ is added to the inner approximation index set $M^i_k$.
To reduce the size of the LP \eqref{eq:ray-intersection-lower-bound-iteration-k}, we can prune the index set $M^i_k$ while ensuring that $M^{i*}_k \subset M^i_k$.
One heuristic for pruning $M^i_k$ is to maintain $|M^i_k| \leq M^i_{max}$, where $M^i_{max} \geq l$, by removing the smallest indices in $M^i_k \setminus M^{i*}_k$.
For two and three-dimensional convex sets, i.e., $l = 2$ and $3$, we can remove all nonoptimal indices from $M^i_k$.
In other words, we use $\texttt{Prune}(M^i_k, M^{i*}_k) = M^{i*}_k$ in Alg.~\ref{alg:update}, line~\ref{subalg:update-pruning}.
The theoretical framework for index set pruning is outside the scope of this paper.
\end{remark}

\subsection{Termination Criteria}
\label{subsec:termination-criteria}

The growth distance algorithm always maintains an upper bound $\beta^u$ and a lower bound $\beta^l > 0$ for $\beta^*$, and thus convergence can be determined by the relative gap $\beta^u / \beta^l - 1 \geq \beta^* / \beta^l - 1$.
The growth algorithm is terminated when the relative gap is not greater than $\epsilon_{tol}$ or if the number of iterations exceeds an algorithm parameter $K_{max} > 0$.

\subsection{Growth Distance Algorithm}
\label{subsec:growth-distance-algorithm}

Combining the initialization and update subroutines presented in the previous subsections, we can solve the ray intersection problem \eqref{eq:ray-intersection-definition} iteratively.
Upon optimal termination of the growth distance algorithm at iteration $k^* \leq K_{max}$, we obtain, from  \eqref{eq:conic-combination-to-optimal-solution}, an $\epsilon_{tol}$-optimal solution $(\beta^*_\epsilon, z^*_\epsilon)$ to the ray intersection problem \eqref{eq:ray-intersection-definition}, where
\begin{equation}
\label{eq:ray-intersection-solution}
    \beta^*_\epsilon = \beta^l_{k^*}, \quad z^*_\epsilon = \textstyle\sum\limits_{m \in M^{i*}_{k^*}} \beta^*_\epsilon \nu^*_m z_m.
\end{equation}
We can then retrieve an $\epsilon_{tol}$-optimal solution $(\alpha^*_\epsilon, z^{1*}_\epsilon, z^{2*}_\epsilon)$ for the growth distance problem \eqref{eq:growth-distance-definition} using Prop.~\ref{prop:ray-intersection-problem} as follows:
\begin{equation}
\label{eq:primal-optimal-from-ray-intersection}
    \alpha^*_\epsilon = 1/\beta^*_\epsilon, \quad z^{i*}_\epsilon = \textstyle\sum\limits_{m \in M^{i*}_{k^*}} \beta^*_\epsilon \nu^*_m z^i_m, \quad i \in \{1, 2\},
\end{equation}
where $z^1_m$ and $z^2_m$ are the support points of $\mathcal{C}^1$ and $\mathcal{C}^2$ used to compute $z_m = z^1_m - z^2_m + p$ (see \eqref{eq:minkowski-difference-support-function}) for all iterations $m$.
Note that when the algorithm terminates due to exceeding $K_{max}$ iterations, $(\beta^*_\epsilon, z^*_\epsilon)$ and $(\alpha^*_\epsilon, z^{1*}_\epsilon, z^{2*}_\epsilon)$ are $(\beta^u_{K_{max}}/\beta^l_{K_{max}} {-} 1)$-optimal.
We can also compute the $\epsilon_{tol}$-optimal normal vector, which achieves the upper bound $\beta^u_{k^*}$, as $\lambda_{m_o}$, where $M^o_{k^*} = \{m_o\}$.
The complete growth distance algorithm is stated in Alg.~\ref{alg:growth-distance-algorithm}.
The normal vector $\lambda_k$ is normalized after each update step (see Alg.~\ref{alg:growth-distance-algorithm}, line~\ref{subalg:lambda-normalization}) to address accuracy issues during the support function computation.
Recall that, in iteration $k+1$, $\lambda_k$ is only used to compute $z_{k+1}$, which is invariant to positive scaling of $\lambda_k$, and $v_k$, which is $1$st-order homogeneous with respect to $\lambda_k$.
Thus, \eqref{eq:outer-approximation-update-iteration-k}, and consequently the algorithm, is not affected by the normalization of the normal vector $\lambda_k$.

\algsetup{indent=1em} 
\begin{algorithm}[!t]
\caption{Growth Distance Algorithm}\label{alg:growth-distance-algorithm}
\begin{algorithmic}[1] 
\REQUIRE $p$, $\underline{r}$, $s_v[\mathcal{C}]$, $s_p[\mathcal{C}]$, $\epsilon_{tol}$, $K_{max}$
\STATE \textbf{Function} \texttt{GrowthDistance}():
\STATE \hspace{1em} \textbf{if} $\lVert p \rVert = 0$ \textbf{then} \label{subalg:growth-distance-algorithm-p-zero-norm-check}
\STATE \hspace{2em} $\displaystyle(\alpha^*_\epsilon, z^{1*}_\epsilon, z^{2*}_\epsilon) \gets (0, p^1, p^2)$
\STATE \hspace{2em} \textbf{return} $\displaystyle \alpha^*_\epsilon, (z^{1*}_\epsilon, z^{2*}_\epsilon)$
\STATE \hspace{1em} \textbf{end if}
\STATE \hspace{1em} $\displaystyle((M^o_0, \beta^u_0), (M^i_0, M^{i*}_0, \beta^l_0), \lambda_0) \gets$ \texttt{Initialize}() \label{subalg:growth-distance-algorithm-init}
\STATE \hspace{1em} \textbf{for} $k = 1$ to $K_{max}$ \textbf{do}
\STATE \hspace{2em} \LONGTEXT{206.5pt}{
$\displaystyle((M^o_k, \beta^u_k), (M^i_k, M^{i*}_k, \beta^l_k), \lambda_k) \gets$\\[1pt]
\hspace*{1em} \texttt{Update}$\displaystyle(M^o_{k-1}, (M^i_{k-1}, M^{i*}_{k-1}), \lambda_{k-1})$
}
\STATE \hspace{2em} $\displaystyle \lambda_k \gets \lambda_k / \lVert \lambda_k\rVert_\infty$ \label{subalg:lambda-normalization}
\STATE \hspace{2em} \textbf{if} $\displaystyle(\beta^u_k/\beta^l_k - 1 < \epsilon_{tol})$ \textbf{then}
\STATE \hspace{3em} \textbf{break}
\STATE \hspace{2em} \textbf{end if}
\STATE \hspace{1em} \textbf{end for}
\STATE \hspace{1em} $\displaystyle(k^*, \lambda^*_\epsilon) \gets (k, \lambda_{m_o})$, where $M^o_k = \{m_o\}$
\STATE \hspace{1em} $\displaystyle(\alpha^*_\epsilon, z^{1*}_\epsilon, z^{2*}_\epsilon) \gets$ computed using \eqref{eq:ray-intersection-solution}, \eqref{eq:primal-optimal-from-ray-intersection}
\STATE \hspace{1em} \textbf{return} $\displaystyle \alpha^*_\epsilon, (z^{1*}_\epsilon, z^{2*}_\epsilon), \lambda^*_\epsilon$
\STATE \textbf{End Function}
\end{algorithmic}
\end{algorithm}

\subsection{Algorithm Properties and Extensions}
\label{subsec:algorithm-properties-and-extensions}

In this subsection, we briefly discuss the properties and extensions of the growth distance algorithm.

\subsubsection{Growth Distance Algorithm Properties}
\label{subsubsec:growth-distance-algorithm-properties}

We first summarize the invariance and monotone convergence properties of the growth distance algorithm, Alg.~\ref{alg:growth-distance-algorithm}.

\begin{theorem} (Growth distance algorithm invariants)
\label{thm:growth-distance-algorithm-invariants}
The following properties hold for the growth distance algorithm Alg.~\ref{alg:growth-distance-algorithm} for all iterations $k \geq 0$, under the assumptions of Problem~\ref{problem:growth-distance}:
\begin{enumerate}[label=(\roman*)]
    \item (Primal feasibility) $(\beta, z)$ and $(\alpha, z^1, z^2)$, where
    \begin{subequations}
    \label{eq:primal-feasibility-invariant}
    \begin{align}
        \beta = \beta^l_k, \quad & z = \textstyle\sum\limits_{m \in M^{i*}_k} \beta \nu^*_m z_m, \\
        \alpha = 1/\beta, \quad & z^i = \textstyle\sum\limits_{m \in M^{i*}_k} \beta \nu^*_m z^i_m, \quad i \in \{1, 2\},
    \end{align}
    \end{subequations}
    are feasible for the ray intersection \eqref{eq:ray-intersection-definition} and growth distance \eqref{eq:growth-distance-definition} problems, respectively.
    Moreover, $\beta, \alpha > 0$.

    \item (Dual feasibility) $\lambda_k$ satisfies the following properties:
    \begin{equation}
    \label{eq:dual-feasibility-invariant}
    \begin{gathered}
        \langle \lambda_k, z_{m_1}\rangle = \langle \lambda_k, z_{m_2}\rangle, \quad \forall m_1, m_2 \in M^{i*}_k, \\
        \langle \lambda_{k-1}, p\rangle > 0.
    \end{gathered}
    \end{equation}

    \item (Simplex property) $\{z_m\}_{m \in M^{i*}_k} \subset \mathcal{C}$ is a linearly independent set.
\end{enumerate}
\end{theorem}

\begin{proof}
The proof follows from the exposition in Secs.~\ref{subsec:inner-outer-polyhedral-approximations} to \ref{subsec:polyhedral-approximation-update}.
In particular, primal feasibility follows from the invariance of \eqref{eq:inner-approximation-properties-iteration-k-1}, \eqref{eq:conic-combination-to-optimal-solution}, and Prop.~\ref{prop:ray-intersection-problem}.
Dual feasibility follows from the invariance of \eqref{eq:normal-properties-iteration-k-1}, and the Simplex property follows from the invariance of the second condition in \eqref{eq:inner-approximation-properties-iteration-k-1}.
\end{proof}

While a discussion of the asymptotic convergence properties of the growth distance algorithm is outside the scope of this paper, we can state the following monotone convergence result, the proof of which follows from the discussion in Secs.~\ref{subsec:inner-outer-polyhedral-approximations} to \ref{subsec:polyhedral-approximation-update}.
Note that monotone convergence does not imply asymptotic convergence to the optimal solution.

\begin{theorem} (Monotone convergence)
\label{thm:monotone-convergence}
For the growth distance algorithm Alg.~\ref{alg:growth-distance-algorithm}, under the assumptions of Problem~\ref{problem:growth-distance}, the relative gap $\beta^u_k/\beta^l_k - 1$ is monotonically decreasing.
\end{theorem}

\subsubsection{Numerical Robustness}
\label{subsubsec:numerical-robustness}

A few computational steps in the growth distance algorithm, Alg.~\ref{alg:growth-distance-algorithm}, require careful attention to avoid numerical issues arising from floating-point errors.
In particular, we describe a robust method for the inner approximation update step (see Sec.~\ref{subsubsec:inner-approximation-update}).

Each iteration of the Simplex method requires the computation of the conic coefficients $(\tilde{\nu}_m)_{m \in M^{i*}_k} = Z^{-1}z_k$ and $(\nu^*_m)_{m \in M^{i*}_k} = Z^{-1}p$, where $Z$ is defined in \eqref{eq:simplex-method-iteration-k-normal}, to compute the outgoing index in \eqref{eq:simplex-iteration-k-outgoing-index}.
When the algorithm is executed with a small tolerance value, $\epsilon_{tol}$, the matrix $Z$ can become ill-conditioned, thereby affecting solution accuracy.
Instead of directly evaluating the conic coefficients, we project $z_k$, $p$, and the columns of $Z$ onto the subspace orthogonal to $p$ and compute the barycentric coordinates for use in \eqref{eq:simplex-iteration-k-outgoing-index}.
In other words, we compute the barycentric coordinates
\begin{equation*}
\tilde{\mu}_m = \frac{\tilde{\nu}_m}{\sum_{n \in M^{i*}_k}\tilde{\nu}_n}, \quad \mu^*_m = \frac{\nu^*_m}{\sum_{n \in M^{i*}_k}\nu^*_n},
\end{equation*}
for $m \in M^{i*}_k$.
Note that scaling $(\tilde{\nu}_m)_{m \in M^{i*}_k}$ and $(\nu^*_m)_{m \in M^{i*}_k}$ does not affect the computation of the outgoing index in \eqref{eq:simplex-iteration-k-outgoing-index}.
Then, the barycentric coordinates are robustly computed using a signed-volume method, as described in \cite{montanari2017improving}.

\subsubsection{Warm Start}
\label{subsubsec:warm-start}

When the convex sets $\mathcal{C}^1$ and $\mathcal{C}^2$ are in motion, the algorithm state from a previous growth distance evaluation can be used to warm start the algorithm.

Let $T_g(\mathcal{C}^1)$ represent $\mathcal{C}^1$ after a small displacement, where $g = (\delta p, \delta R) \in SE(3)$, and let, without loss of generality, $\mathcal{C}^2$ be stationary.
Let $\{z^1_m\}_{m \in M^i} \subset \mathcal{C}^1$ and $\{z^2_m\}_{m \in M^i} \subset \mathcal{C}^2$ be the support points corresponding to the optimal inner approximation from the previous growth distance evaluation (see \eqref{eq:primal-optimal-from-ray-intersection}).
Then, we can define the inner polyhedral approximation after the small displacement $g = (\delta p, \delta R)$ as
\begin{equation}
\label{eq:warm-start-inner-approximation}
    \mathcal{P}^i_{ws} := \{\delta Rz^1_m + \delta p - z^2_m\}_{M^i} \subset T_g(\mathcal{C}^1) - \mathcal{C}^2.
\end{equation}
So, $\mathcal{P}^i_{ws}$ can be used to determine an inner approximation for the Minkowski difference set after the displacement $g$ (without using the support functions).
To warm start the algorithm, we use the points in $\mathcal{P}^i_{ws}$ to update the initial inner approximation $\mathcal{P}^i_0$.
This step, carried out after Alg.~\ref{alg:growth-distance-algorithm}, line~\ref{subalg:growth-distance-algorithm-init}, updates the inner approximation $\mathcal{P}^i_0$ using the points in $\mathcal{P}^i_{ws}$ instead of computing the support points of $T_g(\mathcal{C}^1)$ and $\mathcal{C}^2$, using the same update method described in Sec.~\ref{subsubsec:inner-approximation-update}.
The algorithm is then continued using the updated inner approximation $\mathcal{P}^i_0$.
Since the iterative update steps for warm start are the same as those for cold start, the convergence results in Sec.~\ref{subsubsec:growth-distance-algorithm-properties} still hold.

\subsubsection{Boolean Collision Detection}
\label{subsubsec:boolean-collision-detection}

Since the growth distance algorithm maintains upper and lower bounds $\beta^u$ and $\beta^l$ for $\beta^*$, we can terminate the algorithm early when only collision detection is needed.
If $\beta^l_k \geq 1$ at iteration $k$, then we know that $\alpha^* = 1/\beta^* \leq 1/\beta^l_k \leq 1$.
In this case, the convex sets are in collision.
Likewise, the convex sets are separated if $\beta^u_k < 1$ at iteration $k$.
A crucial advantage of our method is that, since the normal vector achieving the upper bound $\beta^u_k$ is continually tracked, we can compute a separating hyperplane between the convex sets when they are separated.
Similarly, when the convex sets are in collision, a point in their intersection can be computed.
Additionally, the algorithm state after a collision detection evaluation can be used for growth distance warm start, and vice versa.

%% file: sections/results.tex
\section{Results}
\label{sec:results}

In this section, we present benchmark results for the growth distance algorithm and compare them with other growth distance methods.
Finally, we illustrate robotics applications of our method.

\subsection{Benchmark Setup}
\label{subsec:benchmark-setup}

The growth distance algorithm, Alg.~\ref{alg:growth-distance-algorithm}, and the warm start and collision detection functionalities described in Sec.~\ref{subsec:algorithm-properties-and-extensions} are implemented in C\texttt{++} using the \textit{Eigen3} library \cite{eigenweb} and compiled using the \texttt{-O3} flag.
The benchmarks are run on a laptop using \textit{Ubuntu 20.04} with an \textit{Intel Core i7-10870H} CPU ($\qty{2.20}{GHz}$).
We set the maximum number of iterations $K_{max} = 100$ and the tolerance $\epsilon_{tol} \approx 1.49 \times 10^{-8}$ (square root of the machine epsilon) for the algorithm.
The growth distance library and the benchmark code can be found in the repositories\footnote{
Growth distance library: \url{https://github.com/HybridRobotics/differentiable-growth-distance} \\
Benchmark code: \url{https://github.com/AkshayThiru/dgd-benchmark}
}.

Our algorithm supports all convex sets (that satisfy the assumptions of Problem~\ref{problem:growth-distance}) with a support function implementation, including convex primitives, mesh sets, and DSFs (see Tab.~\ref{tab:comparison-growth-distance-methods}).
Sets with curved surfaces often require more iterations for convergence, and so we provide separate benchmarks for curved convex primitives and for meshes.
The geometric parameters of all primitive sets are varied over two orders of magnitude (for example, for a cylinder, the radius is varied from $\qty{0.25e-2}{m}$ to $\qty{0.25}{m}$).
So, our benchmark results include highly skewed objects such as needles and sheets.
For the mesh benchmark, we use $27$ different objects from the YCB dataset \cite{calli2017yale}.
The mesh vertices are convexified using Qhull \cite{barber1996quickhull}, and the number of mesh convex hull vertices ranges from $241$ to $4516$.
The mesh support functions are computed using the hill-climbing algorithm~\cite{van2003collision}.

We chose $1000$ random pairs of convex sets (with random geometric parameters) for each benchmark case, and for each pair we sample $100$ random poses.
For a given pair of convex sets and poses, the computation times are averaged over $100$ function calls.
Thus, the statistics for each benchmark case (corresponding to one bar graph in Fig.~\ref{fig:growth-distance-benchmark}) are computed over $10^7$ function calls.
Note that, unlike for minimum distance algorithms such as GJK (see \cite{montaut2024gjk}), the growth distance computation accuracy does not depend on the minimum distance between the convex sets.
For the warm start benchmarks, we randomly displace the convex sets so that the translation and rotation approximately correspond to $1\%$ of each set's size.

We compare our method with the \emph{Inc}, \emph{IE}, \emph{DCol}, and \emph{DCF} algorithms (see Tab.~\ref{tab:comparison-growth-distance-methods}).
We implemented optimized versions of the Incremental (\emph{Inc}) algorithm \cite{ong2000fast} and the Internal Expanding (\emph{IE}) algorithm \cite{zheng2010fast} in C\texttt{++}.
For a fair comparison, most of the optimizations used in the original C implementation of \emph{Inc} were preserved, except for the memory-intensive ones.
Similarly, our \emph{IE} algorithm implementation uses the same optimizations and convergence criteria as our method.
We note that not all methods work for all convex set types and that not all methods have warm-start functionality (see Tab.~\ref{tab:comparison-growth-distance-methods}).

\subsection{Benchmark Results}
\label{subsec:benchmark-results}

\begin{figure}[!t]
\centering
\subfloat[
Growth distance benchmark results for curved primitives and mesh sets.
]{%
\begin{minipage}{\columnwidth}
    \centering
    \includegraphics[width=0.95\columnwidth]{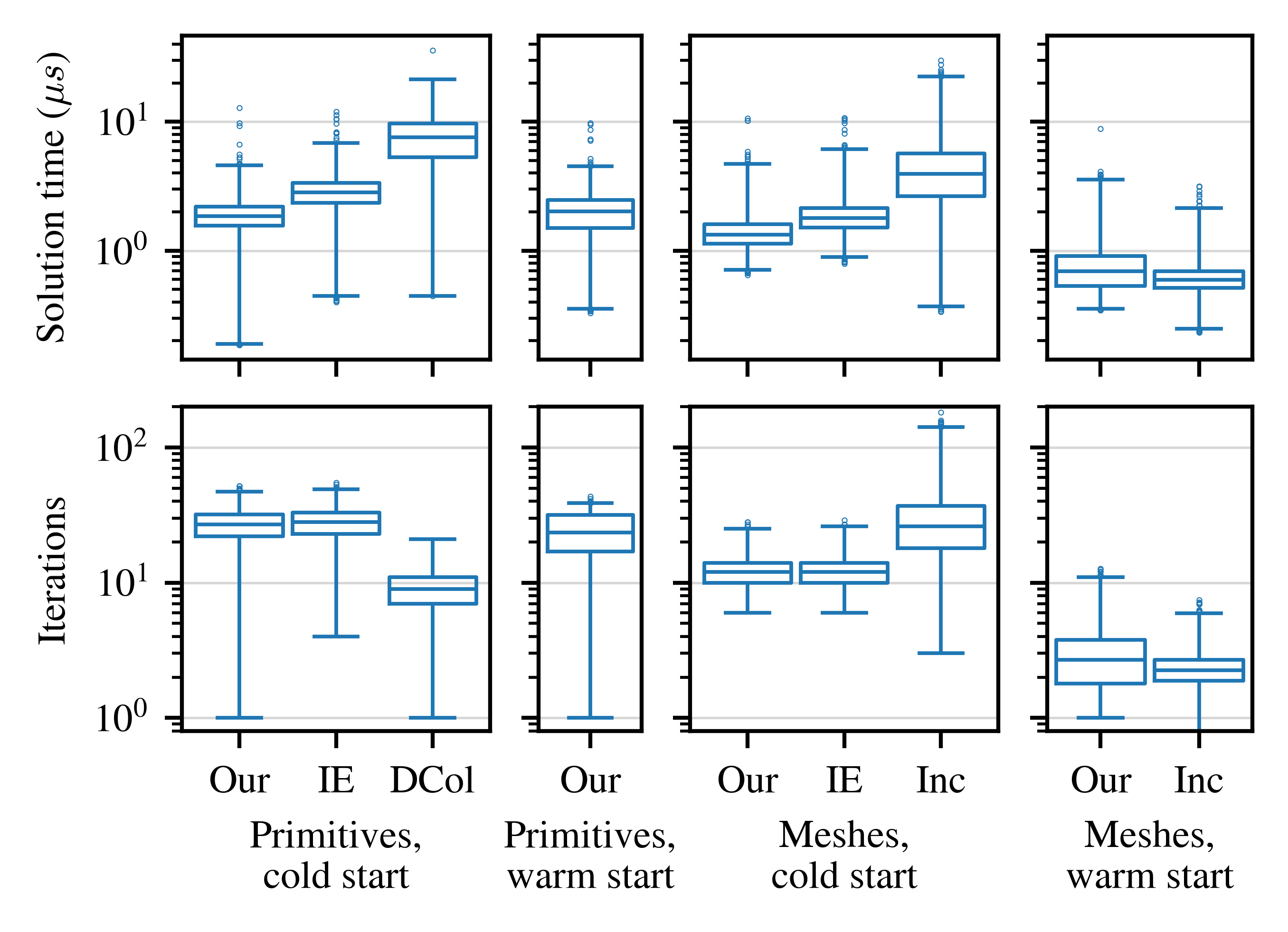}%
    \label{subfig:benchmark-primitives-meshes}
\end{minipage}%
}
\hfil
\subfloat[
Growth distance benchmark results for DSFs in cold start.
]{%
\begin{minipage}{\columnwidth}
    \centering
    \includegraphics[width=0.95\columnwidth]{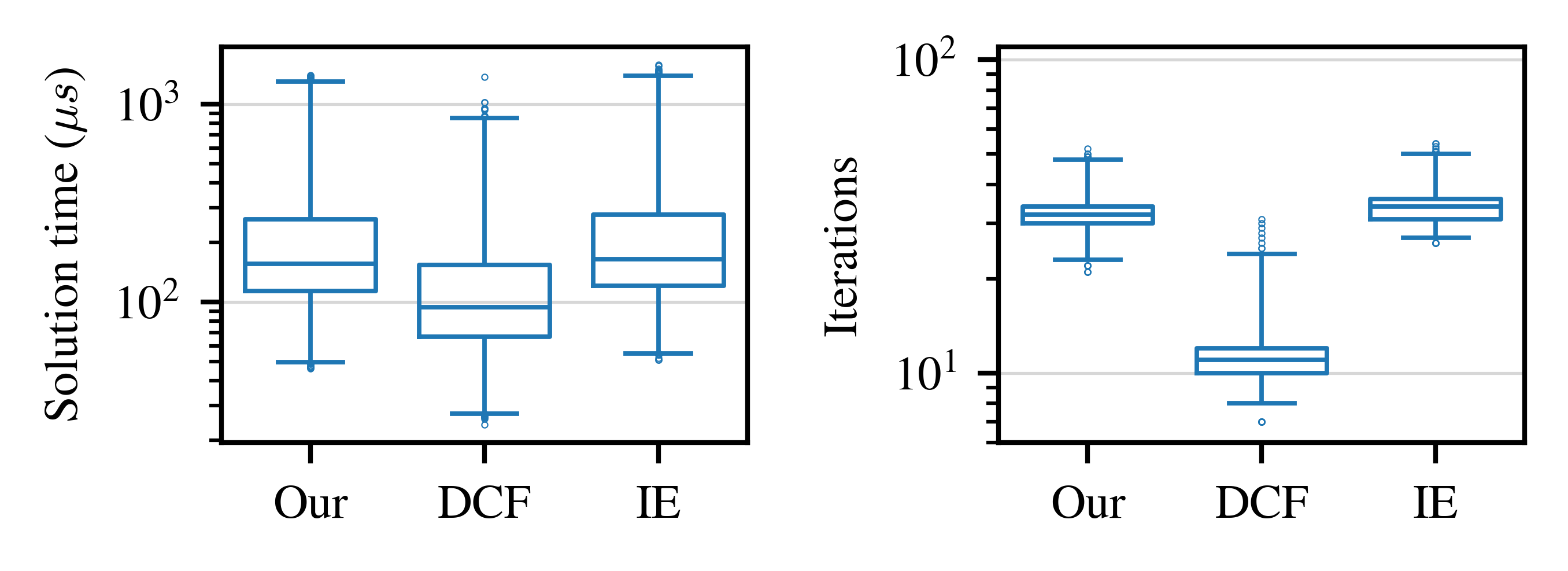}%
    \label{subfig:benchmark-dsfs}
\end{minipage}%
}
\caption{
Growth distance benchmark results for various types of convex sets, with cold and warm starts, across different methods.
For each bar plot, the bars correspond to the $25$-th, $50$-th, and $75$-th percentile values, and the whiskers show the $0.01$ and $99.99$ percentile values.
The outliers are shown as dots.
Our algorithm achieves state-of-the-art solution times for curved primitives and mesh sets (in cold start) across all existing methods.
For DSFs, the hybrid Newton solver DCF proposed in \cite{lee2023uncertain} outperforms our algorithm.
This is primarily because the DCF method requires fewer iterations to converge.
}
\label{fig:growth-distance-benchmark}
\end{figure}


\begin{table*}[!t]
\footnotesize
\setlength\extrarowheight{2pt}
\caption{Solution Times for Growth Distance Computation}
\label{tab:solution-times-benchmark}
\begin{tabularx}{\textwidth}{>{\raggedright\arraybackslash}p{2.5cm}>{\raggedright\arraybackslash}p{0.5cm}CCCCC}
\toprule
\multicolumn{2}{c}{Benchmark case} & \multicolumn{5}{c}{Solution times: (median, $99.99$-th percentile) in microseconds (`$-$' means unsupported)} \\
\cline{3-7}
\multicolumn{2}{c}{(cs: cold start, ws: warm start)} & \emph{IE} \cite{zheng2010fast} & \emph{DCF} \cite{lee2023uncertain} & \emph{Inc} \cite{ong2000fast} & \emph{DCol} \cite{tracy2023differentiable} & \textbf{Our method} \\
\midrule
\multirow{2}{2.5cm}{Curved primitives} & cs & $(2.83, 6.86)$ & $-$ & $-$ & $(7.61, 21.4)$ & $\bm{(1.86, 4.60)}$ \\
& ws & $-$ & $-$ & $-$ & $-$ & $\bm{(2.01, 4.52)}$ \\
\addlinespace[5pt]
\multirow{2}{2.5cm}{Meshes} & cs & $(1.79, 6.15)$ & $-$ & $(3.93, 22.4)$ & $-$ & $\bm{(1.33, 4.70)}$ \\
& ws & $-$ & $-$ & $\bm{(0.60, 2.14)}$ & $-$ & $(0.69, 3.57)$ \\
\addlinespace[5pt]
DSFs \cite{lee2023uncertain} & cs & $(165, 1390)$ & $\bm{(94.4, 848)}$ & $-$ & $-$ & $(156, 1300)$ \\
\bottomrule
\end{tabularx}
\end{table*}

The benchmark results for curved primitive sets and meshes are shown in Fig.~\ref{subfig:benchmark-primitives-meshes}, and for DSFs in Fig.~\ref{subfig:benchmark-dsfs}.
The median and $99.99$-th percentile values of the solution times are tabulated in Tab.~\ref{tab:solution-times-benchmark}.
We observe that across all convex sets (except DSFs), our algorithm outperforms all other methods in cold-start initialization.
For mesh sets in warm-start initialization, our algorithm achieves performance comparable to the \emph{Inc} algorithm.
Compared to the \emph{DCol} algorithm, the per-iteration computation time compensates for the larger number of iterations.
The performance advantage over the \emph{IE} algorithm is more significant for primitive sets than for meshes and DSFs.
This observation can be explained by our algorithm's use of a more computationally efficient Simplex update step.
The \emph{DCF} algorithm uses a hybrid Newton method, resulting in convergence in fewer iterations and, thus, faster solution times.
However, approximating meshes using DSFs (as is done in some \emph{DCF} applications, see \cite{lee2023uncertain}) results in solution times roughly two orders of magnitude larger.
Thus, the \emph{DCF} algorithm is more feasible for sets that can be directly represented as DSFs.

For curved primitive sets, warm-start initialization provides negligible performance benefits over cold-start initialization.
This is because, for curved primitive sets, the support function computations are inexpensive and the transformed inner approximation vertices obtained using warm start (see \eqref{eq:warm-start-inner-approximation}) generally don't lie on the Minkowski difference set boundary.
On the other hand, significant benefits are obtained for meshes due to the flat nature of the boundary of the Minkowski difference set.
We note that the primal warm start technique described in Sec.~\ref{subsubsec:warm-start} is only for the inner polyhedral approximation.
A similar dual warm start technique for the outer polyhedral approximation yielded results comparable to the primal warm start (results for the dual warm start are omitted).

\begin{figure}[!t]
\centering
\includegraphics[width=0.95\columnwidth]{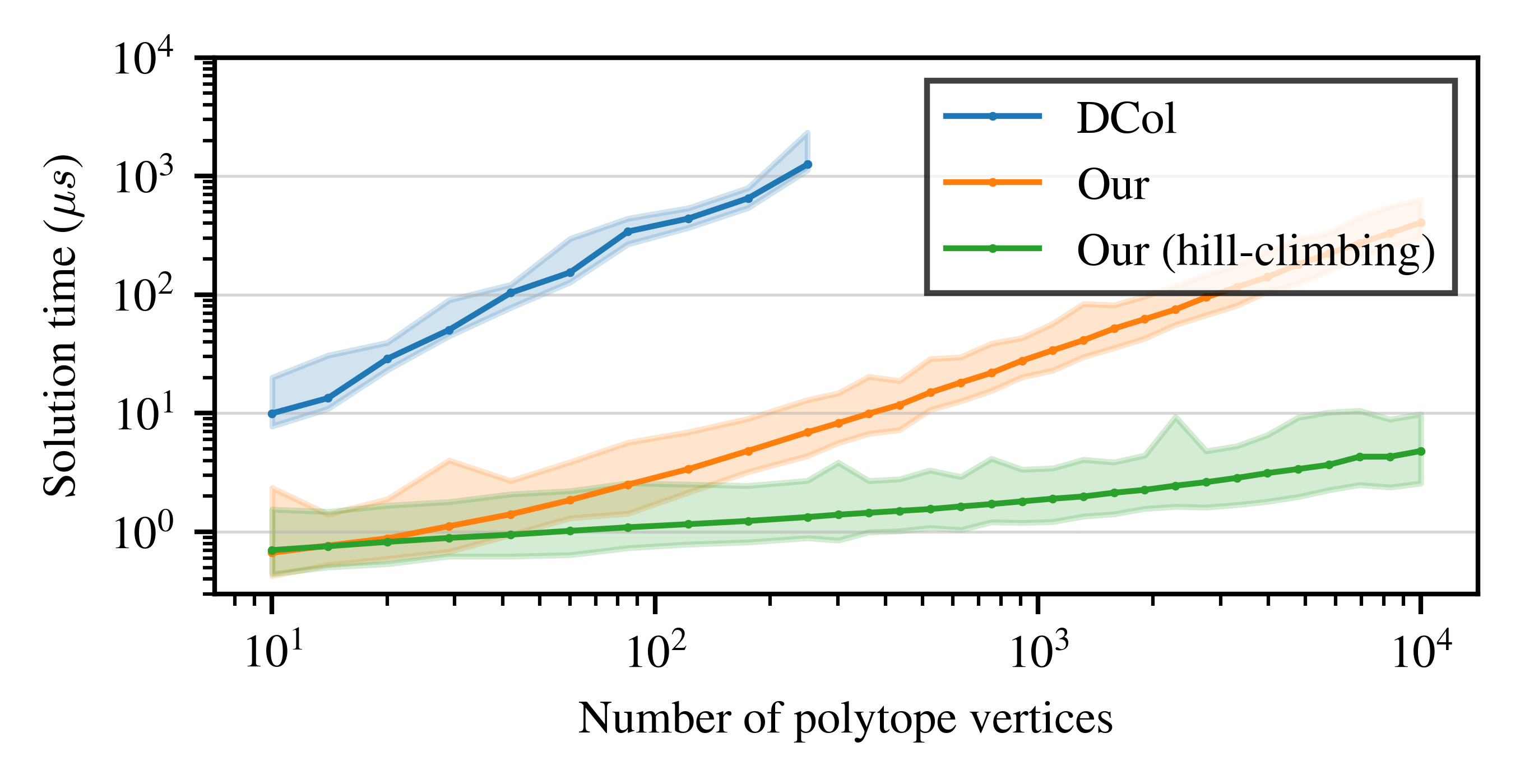}%
\caption{
Solution times for growth distance computation using the interior-point method \emph{DCol}~\cite{tracy2023differentiable} and our method, for polytopes with different numbers of vertices.
For a given number of polytope vertices, the vertices are randomly generated from an ellipsoid surface with an aspect ratio equal to $10$.
The lines show the median solution times, and the shaded regions correspond to the minimum and maximum values.
As the number of polytope vertices increases, the solution time for \emph{DCol} (in blue) increases drastically.
On the other hand, our method demonstrates (roughly) linear growth in solution time (in orange).
When using the hill-climbing algorithm~\cite{van2003collision} to accelerate support function computation, we achieve sublinear growth in solution time (in green).
}
\label{fig:growth-distance-dcol-polytopes}
\end{figure}

Fig.~\ref{fig:growth-distance-dcol-polytopes} compares the solution time of our method (with and without the hill-climbing algorithm~\cite{van2003collision}) with the optimization-based method \emph{DCol} for computing the growth distance between polytopic sets.
We observe that the computation time for \emph{DCol} increases much faster with the number of vertices.
This is because optimization-based methods require solving KKT systems, which grow linearly with the number of constraints and thus result in polynomial computational complexity.
On the other hand, support function-based methods (such as our method and \emph{IE}) yield linear computational growth.
Furthermore, accelerating the support function computation using hill-climbing results in faster computation times.

\begin{figure}[!t]
\centering
\subfloat[
Convergence rate for curved primitive sets.
]{%
\begin{minipage}{\columnwidth}
    \centering
    \includegraphics[width=0.95\columnwidth]{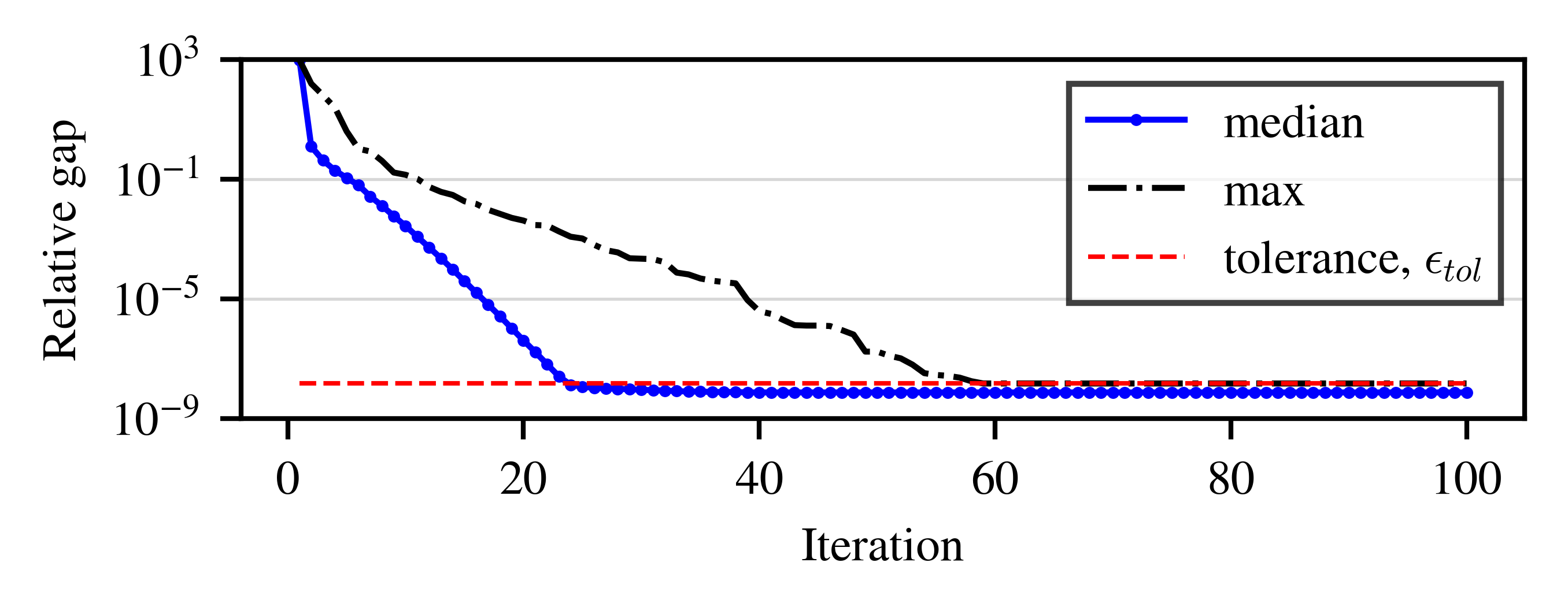}%
    \label{subfig:convergence-rate-primitives}
\end{minipage}%
}
\hfil
\subfloat[
Convergence rate for mesh sets.
]{%
\begin{minipage}{\columnwidth}
    \centering
    \includegraphics[width=0.95\columnwidth]{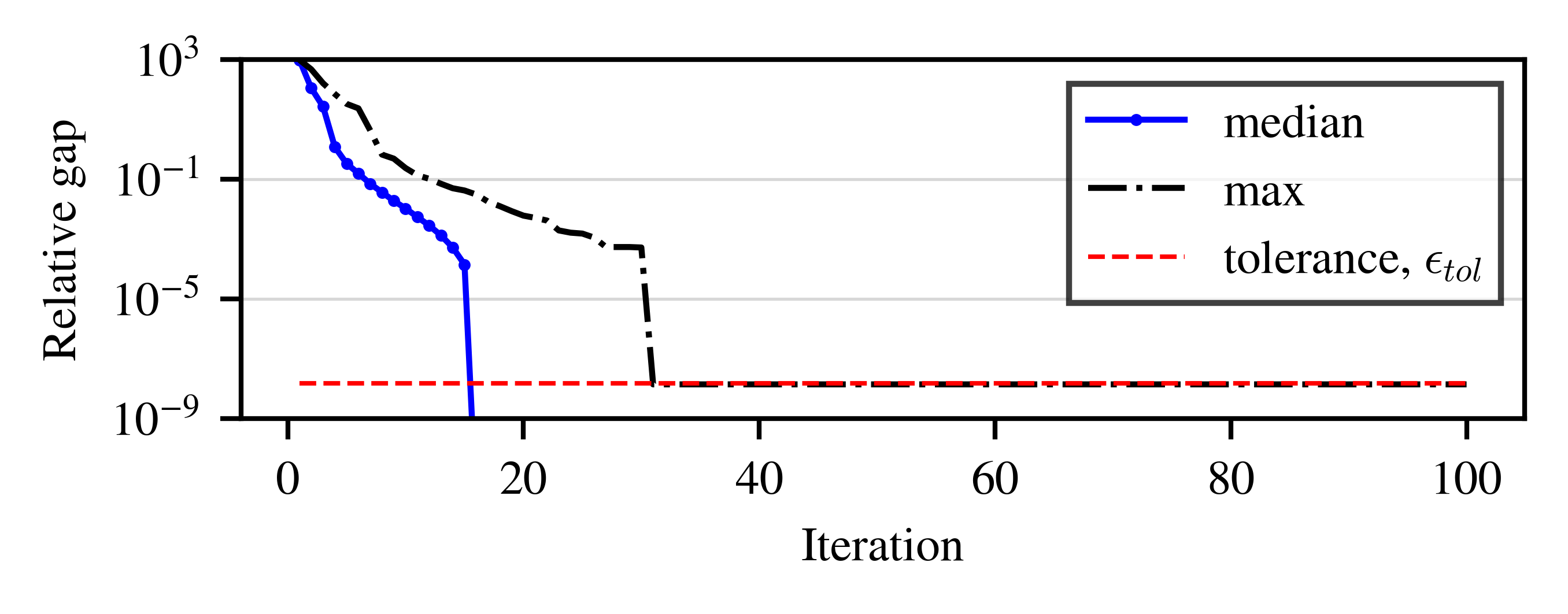}%
    \label{subfig:convergence-rate-meshes}
\end{minipage}%
}
\caption{
Convergence rate results for the growth distance algorithm, Alg.~\ref{alg:growth-distance-algorithm}, for $10^7$ function calls.
The geometric parameters for the primitive sets are varied over two orders of magnitude.
The mesh sets are generated identically to the method described in Fig.~\ref{fig:growth-distance-dcol-polytopes}.
Our method achieves a linear worst-case convergence rate for primitive sets and finite convergence for mesh sets.
}
\label{fig:convergence-rate-benchmark}
\end{figure}

We also show the convergence rate (the log of the relative gap $\beta^u_k/\beta^l_k - 1$ versus the iteration $k$) for the growth distance algorithm, Alg.~\ref{alg:growth-distance-algorithm}, in Fig.~\ref{fig:convergence-rate-benchmark}.
Empirically, we observe that the algorithm converges linearly for primitive sets and converges in a finite number of iterations for mesh sets.

Finally, we summarize some other statistics for our algorithm.
For two-dimensional sets, our algorithm achieves, on average, a solution time of $\qty{0.478}{\mu s}$ with $9.67$ iterations.
Across all growth distance computations for our algorithm in the benchmark (roughly $10^9$ function calls including multiple reruns of the benchmarks), optimality is achieved in all of them and the maximum primal infeasibility error (the error in satisfying the constraint \eqref{subeq:growth-distance-definition-intersection}) is $\qty{6.12e-10}{m}$ for curved primitive sets and $\qty{1.97e-12}{m}$ for meshes.
For the \emph{IE} algorithm, the maximum primal infeasibility error is $\qty{5.95e-7}{m}$ for curved primitive sets and $\qty{1.89e-11}{m}$ for meshes, indicating that our algorithm yields smaller primal infeasibility errors.
Our algorithm achieves a smaller primal infeasibility error compared to the \emph{IE} algorithm due to the numerically robust barycentric coordinate computation method described in Sec.~\ref{subsubsec:numerical-robustness}.

\subsection{Applications}
\label{subsec:applications}

In this subsection, we briefly describe some robotics applications of the growth distance algorithm and discuss its advantages and disadvantages compared to signed distance methods.
We note that the goal of this subsection is to demonstrate the suitability of our method for a range of robotics applications.
An in-depth exploration of the merits of the growth distance relative to the signed distance is left for future work.

\subsubsection{Rigid Body Simulation (Fig.\texorpdfstring{~\ref{fig:physics-simulation}}{ 6})}
\label{subsubsec:rigid-body-simulation}


\begin{figure*}[!t]
\centering
\includegraphics[width=0.99\textwidth]{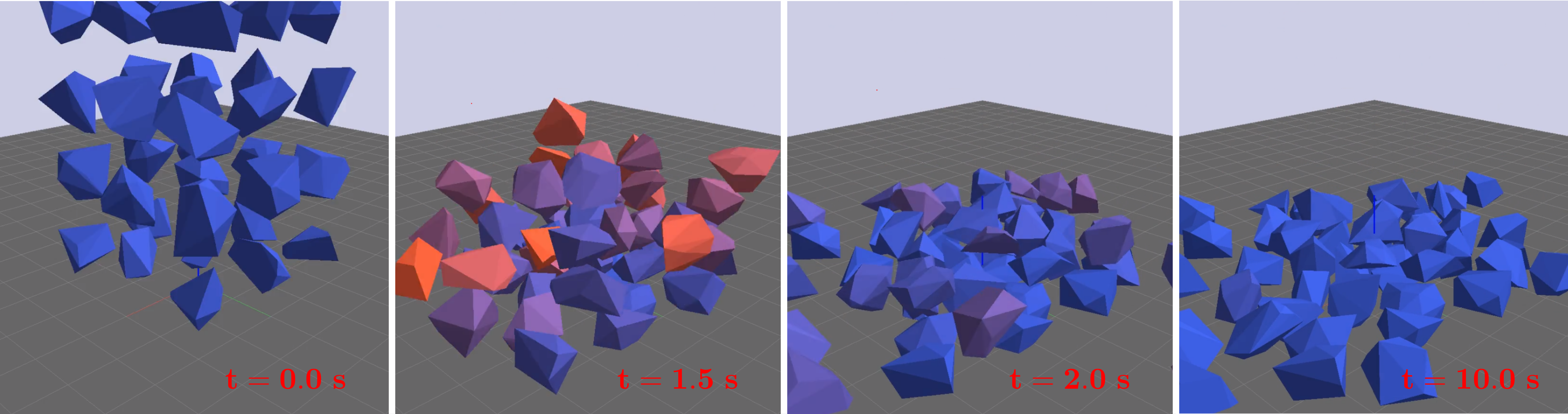}%
\caption{
Multi-rigid-body simulation of $50$ randomly generated polytopes falling on the ground, with our growth distance implementation used as the collision detection and contact frame computation backend in the Bullet Physics SDK.
The polytopes with larger speeds are depicted in red, and the stationary polytopes in blue.
The total kinetic energy of the system reaches zero at the end of $\qty{10}{s}$.
The maximum penetration depth across all collisions is $\qty{1.39}{cm}$, which is similar to the maximum penetration depth achieved using the default signed distance-based collision detector.
}
\label{fig:physics-simulation}
\end{figure*}

In multi-rigid-body simulations, computing contact impulses requires detecting collisions between rigid bodies and computing the contact points and normals.
We consider an example of simulating $50$ polytopes falling to the ground.
We use the Bullet Physics SDK~\cite{coumans2015bullet} (the C\texttt{++} engine) and replace the built-in signed distance-based collision detector with our growth distance implementation.
Following the Bullet recommendations, we set the polytope diameters to roughly $\qty{1}{m}$, with $8$-$20$ vertices per polytope.
We set the collision margins of the polytopes to $\qty{4}{cm}$ (the default value) and ensure their inertias are roughly equal to prevent excessive penetration.
We run the simulation at $\qty{240}{Hz}$ for $\qty{10}{s}$.

Snapshots of the simulation state are depicted in Fig.~\ref{fig:physics-simulation}.
The total kinetic energy of the system reaches $\qty{0.00}{J}$ by $\qty{10}{s}$, with $90.19\%$ of the total initial potential energy dissipated.
The growth distance function is called $7.87 \times 10^5$ times during the $\qty{10}{s}$ duration.
We note that simulating polytopic collisions with the Bullet Physics SDK often results in a positive penetration depth during contact.
The default signed distance-based collision detector and our growth distance implementation result in similar maximum penetration depths of $\qty{1.50}{cm}$ and $\qty{1.39}{cm}$, respectively, across all the collisions (polytope-polytope and polytope-ground).

A disadvantage of using the growth distance for rigid-body simulations is that, unlike the signed distance, the growth distance depends on the center points of the convex sets.
Thus, when the convex sets are strictly penetrating or strictly separated, the contact points and normals differ between the two methods, resulting in different final simulation states.
However, when in contact (i.e., when the growth distance equals one), the contact points and normals for both methods match.
In such cases, minimum distance algorithms such as GJK can suffer from numerical accuracy issues~\cite{montaut2024gjk}, whereas the numerical accuracy of our growth distance algorithm is contact-invariant.
Therefore, the growth distance algorithm can potentially be more numerically stable for rigid-body simulations.

\subsubsection{Motion Planning (Fig.\texorpdfstring{~\ref{fig:manipulator-motion-planning}}{ 7})}
\label{subsubsec:motion-planning}


\begin{figure*}[!t]
\centering
\includegraphics[width=0.99\textwidth]{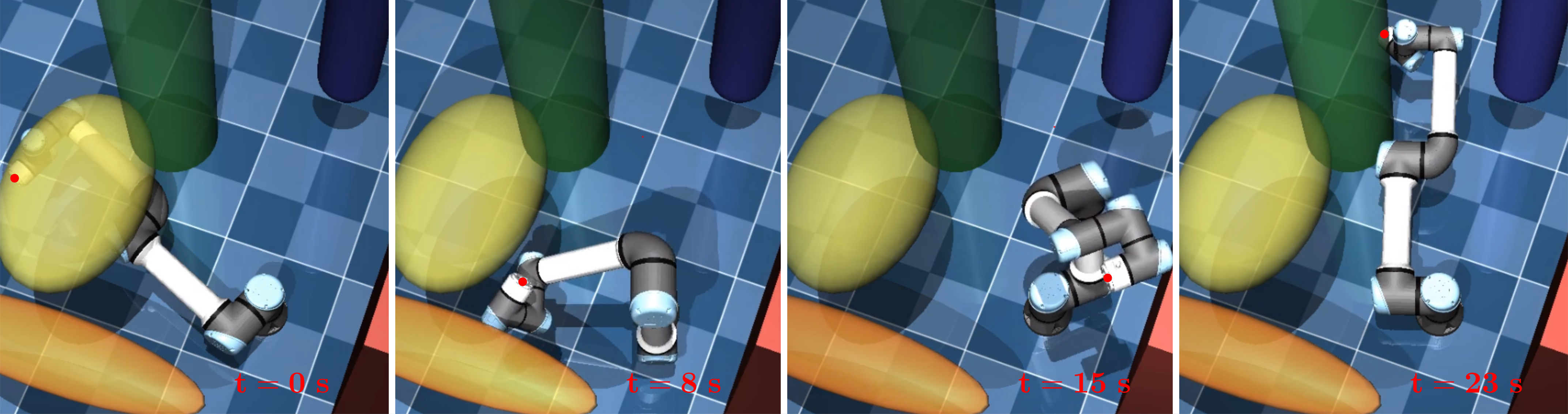}%
\caption{
Motion planning for a manipulator robot in a workspace with five tightly-placed obstacles.
The red dots indicate the end-effector position.
The robot's state validity during motion planning is verified using our growth distance implementation.
}
\label{fig:manipulator-motion-planning}
\end{figure*}

The collision detection functionality described in Sec.~\ref{subsubsec:boolean-collision-detection} can be used for motion planning applications, where it is sufficient to determine whether a given robot state is safe.
We consider motion planning for a UR5e manipulator robot using the RRT implementation in the Open Motion Planning Library (OMPL) \cite{sucan2012open}.
We use the MuJoCo simulator~\cite{todorov2012mujoco} for impedance control of the manipulator joint positions, the UR5e model from MuJoCo Menagerie~\cite{menagerie2022github}, and the Pinocchio library~\cite{carpentier2019pinocchio} for forward kinematics.
The motion planning task is to reach a goal state from an initial state within a workspace with tightly-placed obstacles.
We set the state space resolution for checking state validity to $0.1\%$.

The resulting robot trajectory is shown in Fig.~\ref{fig:manipulator-motion-planning}.
The motion planning problem is solved in $\qty{0.107}{s}$, with $6.03 \times 10^5$ collision detection function calls, and the resulting joint position trajectory is tracked by the position controller in $\qty{23}{s}$.
We note the advantage of using early termination in the growth distance algorithm, which allows us to detect upwards of $5.64 \times 10^6$ collisions per second with cold start.

In robotics applications, there is often significant penetration between geometries during the motion planning and trajectory optimization solution process.
A substantial advantage of using the growth distance for such applications is that it provides a unified measure of set separation and intersection.
This is unlike the signed distance, which is obtained as the solution to a nonconvex optimization problem in the case of set intersection.
Additionally, the computability of the growth distance using a convex optimization problem enables differentiability analysis, a topic that will be explored in depth in future works.

%% file: sections/conclusions.tex
\section{Conclusions}
\label{sec:conclusions}

In this paper, we presented an algorithm to compute the growth distance between convex sets.
We defined the growth distance problem for proper convex sets and stated its equivalent ray intersection problem.
Our algorithm computes the ray intersection solution by maintaining inner and outer polyhedral approximations of the Minkowski difference set.
The inner approximation is updated by incrementally adding support points and using the Simplex method to compute primal feasible solutions.
We also proposed warm start and collision detection functionalities for our algorithm.

Benchmark results show that our algorithm achieves state-of-the-art performance for primitive sets and meshes in cold start, while maintaining low primal infeasibility error.
The improved performance in computing growth distance obtained with our algorithm translates into faster, more accurate collision detection in robotics applications.
Future work on our algorithm will focus on convergence results, algorithmic details for dimensions greater than three, and second-order methods.
Additionally, the differentiability properties of the optimal value and the optimal solution to the growth distance problem will be considered.